%% file: neurips_2021.tex
\documentclass{article}

    \PassOptionsToPackage{numbers}{natbib}

    \usepackage[final]{neurips_2021}

\usepackage[utf8]{inputenc} 
\usepackage[T1]{fontenc}    
\usepackage{hyperref}       
\usepackage{url}            
\usepackage{booktabs}       
\usepackage{amsfonts}       
\usepackage{nicefrac}       
\usepackage{microtype}      
\usepackage{xcolor}         
\usepackage{graphicx}
\usepackage{wrapfig}
\usepackage{subcaption}
\usepackage{csquotes}
\newif\iffull
\fulltrue
\title{A Separation Result Between Data-oblivious and Data-aware Poisoning Attacks}

\author{%
    Samuel Deng\\
    Columbia University\\
    samdeng@cs.columbia.edu\\
   \And
   Sanjam Garg \\
   UC Berkeley and NTT Research\\
   \texttt{sanjamg@berkeley.edu } \\
   \And
     Somesh Jha\\
  University of Wisconsin\\
  \texttt{jha@cs.wisc.edu } \\
   \AND
     Saeed Mahloujifar \\
   Princeton \\
   \texttt{sfar@princeton.edu} \\
   \And
     Mohammad Mahmoody \\
   University of Virginia \\
   \texttt{mohammad@virginia.edu} \\
   \And
  Abhradeep Thakurta \\
   Google Research - Brain Team \\
   \texttt{athakurta@google.com} \\
}
\usepackage{amsfonts} 
\usepackage{amsmath}
\usepackage{amsthm}
\usepackage{paralist}
\usepackage{xcolor}
\usepackage{framed}
\usepackage{graphicx}
\usepackage{float}

\input{macros}

\begin{document}

\maketitle
\begin{abstract}
\input{abstract}
\end{abstract}
\tableofcontents
\input{intro}
\input{games}

\input{seperation}
\input{experiments}

\input{Conclusion_neurips}
\input{Acknolegments}
\bibliography{Biblio/references.bib}
\bibliographystyle{plain}

\newpage
\appendix

\remove{
\medskip

\begin{center} 
{\centering \Large\bf A Separation Result Between Data-oblivious and Data-aware Poisoning Attacks: Supplemental Material}
\end{center}
In the supplemental material, we first provide a more detailed related work section. Then, In Section \ref{sec:defDetails} we discuss some subtle aspects of defining poisoning attacks for data-oblivious and data-aware attacks. In Section \ref{sec:gen_lasso} we provide the details on the borrowed results about the robustness of LASSO that we use to prove our results. Then, in Section \ref{proof:prop} we provide the full proofs of the Theorems stated in the main body. In Section \ref{sec:exp_details_fs} we provide more details and additional experimental results for our feature selection experiments. Finally, in Sections \ref{sec:sepRiskproof} and \ref{ref:class_exp}, we provide some additional theoretical and experimental results (These results are not stated in the main body) about separation of data-oblivious and data-aware adversaries for the case of classification. 
}
\input{seperation_for_population_risk}
\input{experiments_details_PR}

\input{related_work_full}
\input{FullDefs}
\input{general_lasso}

\input{proof_main}
\input{construction}


\end{document}

%% file: macros.tex
\newcommand{\OblFt}{\mathbf{OblFtrSel}}
\newcommand{\FullFt}{\mathbf{AwrFtrSel}}

\newcommand{\OblRisk}{\mathbf{OblRisk}}
\newcommand{\FullRisk}{\mathbf{AwrRisk}}

\newcommand{\LearnFtr}{\mathsf{FtrSelector}}

\newcommand{\elim}{\mathsf{elim}}

\newcommand{\MohNote}[1]{}
\newcommand{\SanNote}[1]{}

\newcommand{\SomNote}[1]{}
\newcommand{\SaeedNote}[1]{}
\newcommand{\AbhrNote}[1]{}
\newcommand{\SamNote}[1]{}

\newcommand{\htet}{\hat{\theta}}

\newcommand{\fhat}[2]{\ifthenelse{\equal{#2}{}}{\hat{f}(#1)}{\ifthenelse{\equal{#2}{0}}{\hat{f}(\emptyset)}{\hat{f}(#1_{\leq #2})}}}
\newcommand{\ftild}[2]{\ifthenelse{\equal{#2}{}}{\tilde{f}(#1)}{\ifthenelse{\equal{#2}{0}}{\tilde{f}(\emptyset)}{\tilde{f}(#1_{\leq #2})}}}
\newcommand{\ftildstar}[2]{\ifthenelse{\equal{#2}{}}{\tilde{f^*}(#1)}{\ifthenelse{\equal{#2}{0}}{\tilde{f^*}(\emptyset)}{\tilde{f^*}(#1_{\leq #2})}}}
\newcommand{\ghat}[2]{\ifthenelse{\equal{#2}{}}{\hat{g}(#1)}{\ifthenelse{\equal{#2}{0}}{\hat{g}(\emptyset)}{\hat{g}(#1_{\leq #2})}}}
\newcommand{\mfix}[2]{\ifthenelse{\equal{#2}{}}{m(#1)}{\ifthenelse{\equal{#2}{0}}{m(\emptyset)}{m(#1_{\leq #2})}}}

\newcommand{\fapp}[2]{\ifthenelse{\equal{#2}{}}{\tilde{f}(#1)}{\ifthenelse{\equal{#2}{0}}{\tilde{f}(\emptyset)}{\tilde{f}(#1_{\leq #2})}}}

\newcommand\aug{\hspace{-2.5mm}\rule[-1mm]{0.1mm}{4.5mm}\hspace{-2.5mm}}
\newcommand\augp{\hspace{-2.5mm}\rule[-2mm]{0.1mm}{5mm}\hspace{-2.5mm}}
\newcommand\augt{\hspace{-2.5mm}\rule[-1mm]{0.1mm}{5mm}\hspace{-2.5mm}}
\newcommand{\concatt}[4]{\begin{bmatrix}
\vspace{-5mm}\\
   #1&\augp&#2\\
   #3&\augt&#4\\
   \vspace{-5mm}\\
    \end{bmatrix}}
\newcommand{\concat}[2]{\begin{bmatrix}
\vspace{-5mm}\\
   #1&
   \aug
   &#2\\
   \vspace{-5mm}\\
    \end{bmatrix}}
\newcommand{\concatv}[2]{\begin{bmatrix}
   #1\\
   #2\\
    \end{bmatrix}}

\newcommand\numberthis{\addtocounter{equation}{1}\tag{\theequation}}

\newcommand{\parag}[1]{\noindent{\bf #1}}
\renewcommand{\paragraph}[1]{\parag{#1}}
\newcommand{\distdist}{\mathfrak{D}}

\newcommand{\ERM}{\mathsf{ERM}}

\newcommand{\X}{\cX} 

\newcommand{\Risk}{\mathsf{Risk}}

\renewcommand{\th}{^\mathrm{th}}

\newcommand{\Exists}{\exists\,}

\newcommand{\aSF}{\mathsf{a}}
\newcommand{\gSF}{\mathsf{g}}
\newcommand{\hSF}{\mathsf{h}}

\newcommand{\avr}[2]{\ifthenelse{\equal{#2}{}}{\aSF({#1})}{\ifthenelse{\equal{#2}{0}}{\aSF(\emptyset)}{\aSF({#1}_{\leq #2})}}}

\newcommand{\avrMax}[2]{\ifthenelse{\equal{#2}{}}{\aSF^*({#1})}{\ifthenelse{\equal{#2}{0}}{\aSF^*(\emptyset)}{\aSF^*({#1}_{\leq #2})}}}

\newcommand{\avrApp}[2]{\ifthenelse{\equal{#2}{}}{\tilde{\aSF}({#1})}{\ifthenelse{\equal{#2}{0}}{\tilde{\aSF}(\emptyset)}{\tilde{\aSF}({#1}_{\leq #2})}}}

\newcommand{\avrAppMax}[2]{\ifthenelse{\equal{#2}{}}{\tilde{\aSF}^*({#1})}{\ifthenelse{\equal{#2}{0}}{\tilde{\aSF}^*(\emptyset)}{\tilde{\aSF}^*({#1}_{\leq #2})}}}

\newcommand{\ArgMax}[2]{\ifthenelse{\equal{#2}{}}{\hSF({#1})}{\ifthenelse{\equal{#2}{0}}{\hSF(\emptyset)}{\hSF({#1}_{\leq #2})}}}

\newcommand{\AppArgMax}[2]{\ifthenelse{\equal{#2}{}}{\tilde{\hSF}({#1})}{\ifthenelse{\equal{#2}{0}}{\tilde{\hSF}(\emptyset)}{\tilde{\hSF}({#1}_{\leq #2})}}}

\newcommand{\gain}[2]{\ifthenelse{\equal{#2}{}}{\gSF(#1)}{\gSF(#1_{\leq #2})}}
\newcommand{\gainMax}[2]{\ifthenelse{\equal{#2}{}}{\gSF^*(#1)}{\gSF^*(#1_{\leq #2})}}
\newcommand{\gainApp}[2]{\ifthenelse{\equal{#2}{}}{\tilde{\gSF}(#1)}{\tilde{\gSF}(#1_{\leq #2})}}
\newcommand{\gainAppMax}[2]{\ifthenelse{\equal{#2}{}}{\tilde{\gSF}^*(#1)}{\tilde{\gSF}^*(#1_{\leq #2})}}

\newcommand{\pr}[2][]{\Pr_{\ifthenelse{\isempty{#1}}{}{{#1}}}\left[{#2}\right]}

\newcommand{\AND}{\mathsf{AND}}

\newcommand{\Loss}{\mathrm{Risk}}

\newcommand{\remove}[1]{}

\newcommand{\se}{\subseteq}

\newcommand{\set}[1]{\left\{ #1 \right\}}

\newcommand{\norm}[1]{\left\lVert#1\right\rVert}

\newcommand{\Learn}{\mathsf{Lasso}}

\newcommand{\Lasso}{\mathsf{Lasso}}

\newcommand{\R}{{\mathbb R}}
\newcommand{\N}{{\mathbb N}}

\newcommand{\Adv}{\mathsf{Adv}}

\newcommand{\cD}{{\mathcal D}}

\newcommand{\cN}{{\mathcal N}}

\newcommand{\cS}{{\mathcal S}}

\newcommand{\cX}{{\mathcal X}}
\newcommand{\cY}{{\mathcal Y}}

\newcommand{\eps}{\varepsilon}

\newcommand{\Exp}{\operatorname*{\mathbb{E}}}
\newcommand{\Ex}{\Exp}

\newcommand{\Sign}{\operatorname{Sign}}

\newcommand{\Supp}{\operatorname{Supp}}

\newcommand{\argmin}{\operatorname*{argmin}}

\newtheorem{theorem}{Theorem}[section]

\newtheorem{proposition}[theorem]{Proposition}
\newtheorem{lemma}[theorem]{Lemma}
\newtheorem{corollary}[theorem]{Corollary}

\newtheorem{definition}[theorem]{Definition}

\newtheorem{remark}[theorem]{Remark}

\newcommand{\sdotfill}{\textcolor[rgb]{0.8,0.8,0.8}{\dotfill}} 

\makeatletter
\def\th@protocol{%
    \normalfont 
    \setbeamercolor{block title example}{bg=orange,fg=white}
    \setbeamercolor{block body example}{bg=orange!20,fg=black}
    \def\inserttheoremblockenv{exampleblock}
  }
\makeatother

\newtheorem{proto}[theorem]{Protocol}
\newtheorem{protoc}[theorem]{Protocol}

\newcommand{\namedref}[2]{#1~\ref{#2}}

\newcommand{\torestate}[3]{%
\expandafter \def \csname BBRESTATE #2 \endcsname{#3}
\theoremstyle{plain}
\newtheorem{BBRESTATETHMNUM#2}[theorem]{#1}
\begin{BBRESTATETHMNUM#2}\label{#2}\csname BBRESTATE #2 \endcsname   \end{BBRESTATETHMNUM#2}
\newtheorem*{BBRESTATETHMNONNUM#2}{\namedref{#1}{#2}}
}

\newcommand{\restate}[1]{\begin{BBRESTATETHMNONNUM#1}[Restated] \csname BBRESTATE #1 \endcsname
\end{BBRESTATETHMNONNUM#1}}


%% file: abstract.tex
Poisoning attacks have emerged as a significant security threat to machine learning algorithms. It has been demonstrated that adversaries who make small changes to the training set, such as adding specially crafted data points, can hurt the performance of the output model. Some of the stronger poisoning attacks require the full knowledge of the training data. This leaves open the possibility of achieving the same attack results using poisoning attacks that do not have the full knowledge of the clean training set.
In this work, we initiate a theoretical study of the problem above. Specifically, for the case of feature selection with LASSO, we show that \emph{full information} adversaries (that craft poisoning examples based on the rest of the training data) are provably stronger than the optimal attacker that is \emph{oblivious} to the training set yet has access to the distribution of the data.  Our separation result shows that the two setting of data-aware and data-oblivious are fundamentally different and we cannot hope to always achieve the same attack or defense results in these scenarios. 

%% file: intro.tex
\section{Introduction} \label{sec:intro}

Traditional approaches to supervised machine learning focus on a benign setting where honestly sampled training data is given to a learner. 
However, the broad use of these learning algorithms in safety-critical applications  makes them targets for sophisticated attackers. Consequently, machine learning has gone through a revolution of studying the same problem, but this time under so-called adversarial settings. Researchers have investigated several types of attacks, including  test-time (a.k.a., evasion attacks to find adversarial examples)~\citep{Szegedy:intriguing,Evasion:TestTime,Adversarial::Harnessing, shafahi2018poison}, training-time attacks (a.k.a., poisoning or causative attacks)~\citep{barreno2006can,biggio2012poisoning,papernot2016towards}, backdoor attacks~\citep{turner2018clean,gu2017badnets},  
membership inference attacks \citep{shokri2017membership},
etc. In response, other works have put forth several defenses~\citep{papernot2016distillation,madry2017towards,biggio2018wild} followed by adaptive attacks~\citep{carlini2017towards,athalye2018obfuscated, tramer2020adaptive} that circumvent some of the proposed defenses. Thus, developing   approaches that are based on solid theoretical foundations (that prevent further adaptive attacks) has stood out as an important area of investigation. 

\parag{Poisoning Attacks.} In a poisoning attack, an adversary changes a training set $\cS$ of examples into a ``close'' training set $\cS'$ (The difference is usually measured by Hamming distance; i.e., the number of examples injected and/or removed.). Through these changes, the goal of the adversary, generally speaking, is to degrade the ``quality'' of the learned model, where quality here could be interpreted in different ways. In a recent industrial survey~\cite{kumar2020adversarial}, poisoning attacks were identified as the most important threat model against applications of machine learning. The main reason behind the importance of poisoning attacks are the feasibility of performing the attack for adversary. As the data is usually gathered from multiple sources, the adversary can perform the poisoning attacks by corrupting one of the sources. Hence, it is extremely important to fundamentally understand this threat model. In particular, we need to investigate the role of design choices that are made in both poisoning attacks and defenses. 

\parag{Does the attacker know the training data?} The role of knowledge of the clean training set is one of the less investigated aspects of poisoning attacks. Many previous work on theoretical analysis of poisoning attacks implicitly, or explicitly, assume that the adversary has  full knowledge of the training data $\cS$ before choosing what examples to add or delete from $\cS$~\cite{koh2018stronger,steinhardt2017certified,mahloujifar2018curse, suya2020model}. 
In several natural scenarios, an adversary might not have access to the training data before deciding on how to tamper with it. This has led researchers to study poisoning attacks that do not use the knowledge of the training set to craft the poison points. In this work, we explore the following question:
\begin{displayquote} \textit{What is the role of the knowledge of  training set in the success of poisoning adversaries? Can the knowledge of training set help the attacks? Or alternatively, can hiding the training set from adversaries help the defenses?} \footnote{This question was independently asked as an open question in the concurrent survey of Goldblum et al. \cite{goldblum2020data}.}
 \end{displayquote}

In this work, as a first step to understand this question, we show a separation result between data-oblivious and data-aware poisoning adversaries. In particular, we show that there exist a learning setting (Feature selection with LASSO on Gaussian data) where poisoning adversaries that know the distribution of data but are oblivious to specific training samples that are used to train the model are provably weaker than the adversaries with the knowledge of both training set and the distribution. To the best of our knowledge, this is the first separation result for poisoning attacks.

\paragraph{Implications of our separation result:}  Here, we mention some implications of our separation result.
\begin{itemize}
    \item \textbf{Separation of threat models:} The first implication of our result is the separation of data-oblivious and data-aware poisoning threat models. Our result shows that data-oblivious attacks are strictly weaker than data-aware attacks. In other words, it shows that we cannot expect the defenses to have the same effectiveness in both scenarios. This makes the knowledge of data a very important design choice that should be clearly stated when designing defenses or attacks.
    \item \textbf{Possibility of designing new defenses:} Although data-oblivious poisoning is a weaker attack model, it might still be the right threat model for many applications. For instance, if data providers use cryptographically secure multi-party protocols to train the model \cite{wagh2019securenn}, then each participant can only observe their own data. Note that each party might still have access to some data pool from the true distribution of training set and that still fits in our data-oblivious threat model. In these scenarios, it is natural to use defenses that are only secure against data-oblivious attacks.  Our results shows the possibility of designing defense mechanisms that leverage the secrecy of training data and can provide much stronger security guarantees in this threat mode. In particular, our result shows the provable robustness of LASSO algorithm in defending against data-oblivious attacks. 
    
    Note that this approach is distinct from the demoted notion of ``security through obscurity'' as the attacker knows every detail of the algorithm as well as the data distribution. The only unknown to the adversary is the randomness involved in the process of sampling training examples from the training distribution. This is exactly similar to how secret randomness helps security in cryptography. 
    \item \textbf{A new motive for privacy:} privacy is often viewed as a utility for data owners in the machine learning pipeline. Due to the trade-offs between privacy and the efficiency/utility, data-users often ignore the privacy of data owners while doing their analysis, especially when there is no incentive to enforce the privacy of the learning protocol. The possibility of improving the security against poisoning attacks by enforcing  the (partial) data-obliviousness of the adversary could create a new incentive for keeping training datasets secret. Specifically, the users of data would now have more motivation to try to keep training dataset private, with the goal of securing their models against poisoning and increasing their utility in scenarios where part of data is coming from potentially malicious sources.
\end{itemize}




\subsection{Our Contributions}
In this work, we provide theoretical evidence that obliviousness of attackers to the training data can indeed help robustness against poisoning attacks. In particular, we provide a provable difference between: (i) an adversary that is aware of the training data as well as the distribution of training data,  before launching the attack (data-aware adversary) and (ii) an adversary that only knows the distribution of training data and does not know the specific clean examples in the training set (data-oblivious adversary). 


We start by formalizing what it means mathematically for the poisoning adversary to be data-oblivious or data-aware. 

\parag{Separations for feature selection with Lasso.} 
We then prove a separation theorem between the data-aware and  data-oblivious poisoning threat models in the context of \emph{feature selection}.
We study data-aware and data-oblivious attackers against the Lasso estimator and show that if certain natural properties holds for the distribution of dataset, the power of optimal data-aware and data-oblivious poisoning adversaries differ significantly. 

We emphasize that in our data-oblivious setting, the adversary \emph{fully knows} the \emph{data distribution}, and hence it implicitly has access to a lot of auxiliary information about the data set, yet the very fact that it does not know the \emph{actual} sampled dataset makes it harder for adversary to achieve its goal.


\parag{Experiments.} To further investigate the power of data-oblivious and data-aware attacks in the context of feature selection, we experiment on synthetic datasets sampled from Gaussian distributions, as suggested in our theoretical results. Our experiments confirm our theoretical findings by showing that  the power of data-oblivious and poisoning attacks differ significantly.  Furthermore, we experimentally evaluate the power of \emph{partially-aware} attackers who only know part of the data. These experiments show the gradual improvement of the attack as the knowledge of data grows.

In our experimental studies we go beyond Gaussian setting and show that the the power of data-oblivious attacks could be significantly lower on real world distributions as well. In our experiments, sometimes (depending on the noise nature of the dataset), even an attacker that knows $20\%$ of the dataset cannot have much of improvement over an oblivious attacker.

\parag{Separation for classification.} In addition to our main results  in the context of feature selection, in this work, we also take initial steps to study the role of adversary's knowledge (about the data set) when the goal of the attacker is to increase the risk of the produced model in the context of classification. These results are presented supplemental material (Section \ref{sec:sepRiskproof} and \ref{ref:class_exp}). 
\subsection{Related Work}
Here, we provide a short version of related prior work. A more comprehensive description of previous work has been provided in Appendix \ref{full_related_work} where we also categorize the existing attacks into data-aware and data-oblivious categories.


Beatson et al.~\cite{beatson2016blind} study ``Blind'' attackers against machine learning models that do not even know the distribution of the data. They show that poisoning attacks could be successful in such a restricted setting by studying the minimax risk of learners. They also introduced ``informed'' attacks that see the data distribution, but not the actual training samples and leave the study of these attacks to future work. Interestingly, the ``informed'' setting of \citep{beatson2016blind} is equivalent to the ``oblivious'' setting in our work.

Xiao et al.~\cite{xiao2015feature} empirically examine the robustness of feature selection in the context of poisoning attacks, but their measure of stability is across sets of features. We are distinct in  that our paper studies the effect of data-oblivious attacks on \textit{individual} features and with provable guarantees.


We distinguish our work with another line of work that studies the computational complexity of the attacker \citep{mahloujifar2018can,garg2019adversarially}. Here, we study the ``information complexity'' of the attack; namely, what information the attacker needs to succeed in a poisoning attack, while those works study the \emph{computational resources} that a poisoning attacker needs to successfully degrade the quality of the learned model. Another recent exciting line of work that studies the computational aspect of robust learning in poisoning contexts, focuses on the computational complexity of the \emph{learning} process itself \citep{diakonikolas2016robust,lai2016agnostic,charikar2017learning,diakonikolas2017statistical,diakonikolas2018list,diakonikolas2018sever,prasad2018robust,diakonikolas2018efficient}, and other works have studied the same question about the complexity of the learning process for evasion attacks \citep{bubeck2018adversarial1,bubeck2018adversarial2,degwekar2019computational}. Furthermore,   our work deals with information complexity and  is distinct from  works that study the impact of the training set (e.g., using clean labels) on the success of poisoning \cite{shafahi2018poison,zhu2019transferable,suciu2018does,turner2018clean}.

Our work's motivation for data secrecy  might seem similar to other works that leverage privacy-preserving learning (and in particular differential privacy \citep{dinur2003revealing,TCC:DMNS06,dwork2008differential}) to limit the power of poisoning attacks by making the learning process less sensitive to poison data \citep{ma2019data}. However, despite seeming similarity, what we pursue here is fundamentally different. In this work, we try to understand the effect of keeping the data secret from adversaries. Whereas the robustness guarantees that come from differential privacy has nothing to do with secrecy and hold even if the adversary gets to see the full training set (or even select the whole training set in an adversarial way.).

We also point out some separation results in the context of adversarial examples. The work of Bubeck et al.~\cite{bubeck2018adversarial} studies the separation in the power of \emph{computationally bounded} v.s. \emph{computationally unbounded} learning algorithms in learning robust model. Tsipras et al.~\cite{tsipras2018robustness} studies the separation between \emph{benign accuracy} and \emph{robust accuracy} of classifiers showing that they can be even at odds with each other. Schmidt et al.~\cite{schmidt2018adversarially} show the separation between sample complexity of learning algorithms in training an adversarially robust model versus a model with high benign accuracy. Garg et al.~\cite{garg2019adversarially} separate the notions of \emph{computationally bounded} v.s. \emph{computationally unbounded} attacks in successfully generating adversarial examples. Although all these results are only proven for few (perhaps unrealistic) settings, they still significantly helped the understanding of adversarial examples.

As opposed to the data poisoning setting, the question of adversary's (adaptive) knowledge was indeed previously studied in the line of work on adversarial examples \citep{Adversarial:Old,Adversarial:models,Szegedy:intriguing}. 
In a test time evasion attack the adversary's goal is to find an adversarial example, the adversary knows the input $x$ \emph{entirely} before trying to find a close input $x'$   that is misclassified. So, this adaptivity aspect already differentiates adversarial examples from random noise. 

%% file: games.tex
\section{ Defining Threat Models: Data-oblivious and Data-aware Poisoning} \label{sec:Data-obliviousVsFullInfo}
In this section, we formally define the security games of  learning systems under \emph{data-oblivious} poisoning attacks. It is common in cryptography to define security model based on a game between an adversary and a challenger \citep{KatzLiBook}. Here, we use the same approach and introduce  game based definitions for data-oblivious and data-aware adversaries. 

\paragraph{Feature selection.} The focus of this work is mostly on the feature selection  which is a significant task in machine learning. In a feature selection problem, the learning algorithm wants to discover the relevant features that determine the ground truth function. For example, imagine a dataset of patients with many features, who suffer from a specific disease with different levels of severity. One can try to find the most important features contributing to the severity of the disease in the context of feature selection. Specifically, the learners' goal is to recover a vector $\theta^* \in \R^d$ whose non-zero coordinates determine the relevant features contributing to the disease. In this scenario, the goal of the adversary is to deceit the learning process and make it output a model $\hat{\theta}' \in \R^d$ with a different set of non-zero coordinates.  
As motivation for studying feature selection under adversarial perturbations, note that the non-zero coordinates of the learned model 
could be related to a sensitive subject. For example, in the patient data example described in the introduction, the adversary might be a pharmaceutical institute who tries to imply that a non-relevant feature is contributing to the disease, in order to advertise for a specific medicine.  

We start by separating the \emph{goal} of a poisoning attack from \emph{how} the adversary achieves the goal. The setting of an \emph{data-oblivious} attack deals with the latter, namely it is about how the attack is done, and this aspect is orthogonal to the goal of the attack. In a nutshell, many previous works on data poisoning deal with increasing the population risk of the produced model (see Definition \ref{security:PopRisk} below and Section \ref{sec:defDetails} for more details and variants of such attacks). In a different line of work, when the goal of the learning process is to recover a set of features (a.k.a., model recovery) the goal of an attacker would be defined to counter the goal of the feature selection, namely to add or remove features from the correct model.

In what follows, we describe the security games for a feature selection task. We give this definition for a basic reference setting in which the data-oblivious attacker injects data into the data set, and its goal is to change the selected features. (See Section \ref{sec:defDetails}  for more variants of the attack.)  Later, in Section~\ref{sec:seperation} we will see how to construct problem instances (by defining their data distributions) that provably separate the power of data-oblivious attacks from data-aware ones.

\paragraph{Notation.} We first define some useful notation. For an arbitrary vector $\theta\in \R^d$ we use $\Supp(\theta)=\set{i \colon \theta_i\neq 0}$, we denote the set of (indices of) its non-zero coordinates. We use capital letters (e.g $X$) to denote sets and calligraphic letters (e.g. $\cX$) to denote distributions. $(\cX,\cY)$ denotes the joint distribution of $\cX$ and $\cY$ and $\cX_1\equiv \cX_2$ denotes the equivalence of two distributions $\cX_1$ and $\cX_2$. We use $\norm{\theta}_2$ and $\norm{\theta}$ to denote the $\ell_2$ and $\ell_1$ norms of $\theta$ respectively. For two matrices $X\in R^{n\times d}$ and $Y\in R^{n\times 1}$, we use $\concat{X}{Y}\in R^{n\times (d+1)}$ to denote a set of $n$ regression observations on feature vectors $X_{i\in [n]}$ such that $Y_i$ is the real-valued observation for $X_i$. For two matrices $X_1\in \R^{n_1 \times d}$ and $X_2 \in \R^{n_2 \times d}$, we use $\concatv{X_1}{X_2}\in \R^{(n_1 + n_2)\times d}$ to denote the concatenation of $X_1$ and $X_2$. Similarly, for two set of observations $\concat{X_1}{Y_1}\in \R^{n_1\times (d+1)}$ and $\concat{X_2}{Y_2}\in \R^{n_2\times (d+1)}$, we use $\concatt{X_1}{Y_1}{X_2}{Y_2} \in \R^{(n_1+n_2)\times(d+1)}$ to denote the concatenation of $\concat{X_1}{Y_1}$ and $\concat{X_2}{Y_2}$. For a security game $G$ and an adversary $A$ we use $\Adv(A,G)$ (advantage of adversary $A$ in game $G)$ to denote probability of adversary $A$ winning the security game $G$, where the probability is taken over the randomness of the game and adversary.

Since the security games for data-aware and data-oblivious games are close, we use Definition \ref{security:Feature} below for both, while we specify their exact differences.

\begin{definition}[Data-oblivious and data-aware data injection poisoning for feature selection] \label{security:Feature} 
We first describe the \emph{data-oblivious} security game  between a challenger $C$ and an adversary $A$. The game is parameterized by the adversary's budget $k$ and the training data $\cS=\concat{X}{Y}$ which is a matrix $X$ and a set of labels $Y$, and the feature selection algorithm $\LearnFtr$. 

\noindent~~~$\OblFt(k, \cD,\LearnFtr,n)$.
\begin{compactenum}

    \item Knowing the algorithm $\LearnFtr$ and distribution $\cD$ supported on $\R^{d+1}$, and given $k$ as input, the adversary $A$ generates a poisoning dataset  $\concat{X'}{Y'} \in [-1,1]^{k\times (d+1)}$ of size $k$ such that each row has $\ell_1$ norm at most 1  and sends it to $C$.
    
    \item $C$ samples a dataset $\concat{X}{Y}\gets \cD^n$
    \item $C$ recovers models $\hat{\theta}=\LearnFtr(\concat{X}{Y})$
using the clean data and $\hat{\theta}'=\LearnFtr\left(\concatt{X}{Y}{X'}{Y'}\right)$   using the  {poisoned} data.
    \item Adversary wins if $\Supp(\hat{\theta})\neq \Supp(\hat{\theta}')$, and we use the following notation to denote the winning:
    $$ \OblFt(A, k, \cD,\LearnFtr,n)=1. $$
\end{compactenum}
In the  security game for \emph{data-aware} attackers, all the steps are the same as above, except that the order of steps 1 and 2 are different. Namely, challenger first samples and sends the dataset to adversary.

\noindent~~~$\FullFt(k, \cD,\LearnFtr,n)$.
\begin{compactenum}
\item $C$ samples   $\concat{X}{Y}\gets \cD^n$ and sends it $A$.
\item Knowing the algorithm $\LearnFtr$ and distribution $\cD$ supported on $\R^{d+1}$, the dataset $\concat{X}{Y}$, and given $k$ as input, the adversary $A$ generates a poisoning dataset  $\concat{X'}{Y'} \in [-1,1]^{k\times (d+1)}$  of size $k$ such that each row $\concat{X'}{Y'}$ has $\ell_1$ norm at most 1 and sends it to $C$.
    \item $C$ recovers models $\hat{\theta}=\LearnFtr(\concat{X}{Y})$
using the clean data and $\hat{\theta}'=\LearnFtr\left(\concatt{X}{Y}{X'}{Y'}\right)$   using the  {poisoned} data.
    \item Adversary wins if $\Supp(\hat{\theta})\neq \Supp(\hat{\theta}')$, and we use the following notation to denote the winning:
$$ \FullFt(A, k, \cD,\LearnFtr,n)=1. $$
\end{compactenum}
\end{definition}


\paragraph{Variations of security games for Definition~\ref{security:Feature}.} Definition~\ref{security:Feature} is written only for the case of feature-flipping attacks by only injecting poison data. 
One can, however, envision variants by changing the adversary's goal and how it is doing the poisoning attack. In particular, one can define more specific goals for the attacker to violate the feature selection, by  aiming to add or remove non-zero coordinates to the recovered model compared to the ground truth.\footnote{In fact, one can even define \emph{targeted} variants in which the adversary even picks the feature that it wants to add/remove or flip.} In addition, it is also possible to change the method of the adversary to employ data elimination or substitution attacks. 

One can also imagine \emph{partial-information} attackers who are exposed to a fraction of the data set $\cS$ (e.g., by being offered the knowledge of a  randomly selected $p$ fraction of the rows of $[X|Y]$. Our experiments deal with this very setting.

\paragraph{Why bounding the norm of the poison points?} When bounding the number of poison points, it is important to bound  the norm of the poisoning points according to some threshold (e.g. through a clipping operation) otherwise a single poison point can have infinitely large effect on the trained model. By bounding the $\ell_1$ norm of the poison data, we make sure that a single poison point has a bounded effect on the objective function and cannot play the role of a large dataset. We could remove this constraint from the security game and enforce it in the algorithm through a clipping operation but we keep it as a part of definition to emphasize on this aspect of the security game. Note that in this work we always assume that the data is centered around zero. That is why we only use a constraint on the norm of the poison data points.  However, the security game could be generalized by replacing the $\ell_2$ norm constraint with an arbitrary filter $F$ for different scenarios.

\paragraph{Why using $\hat{\theta}$ instead of $\theta$.}
Note that in security games of Definition \ref{security:Feature} we do not use the \emph{real} model $\theta$ (or more accurately its set of features $\Supp(\theta)$), but rather we work with $\Supp(\hat{\theta})$. That is because, we will work with promised data sets for which $\LearnFtr$ provably recovers the true set of features $\Supp(\hat{\theta}) = \Supp({\theta})$. This could be guaranteed, e.g., by putting conditions  on the  data. 

\paragraph{Why injecting the poison data to the end?} Note that in security games of Definition \ref{security:Feature}, we are simply injecting the poison examples to the \emph{end} of the training sequence defined by $X,Y$, instead of asking the adversary to pick their locations. That is only for simplicity, and the definition is implicitly assuming that the feature selection algorithm is symmetric with respect to the order of the elements int the data set (e.g., this is so for  $\Lasso$ estimator). However, one can generalize the definition directly to allow the adversary to pick the specific location of the added elements. 

%% file: seperation.tex
\section{Separating   Data-oblivious and Data-aware Poisoning for Feature Selection} \label{sec:seperation}
In this section, we provably demonstrate that the power of data-oblivious and data-aware adversaries could  significantly differ. Specifically, we study the power of poisoning attacks on feature selection.

\parag{Feature selection by the Lasso estimator.} We work in the feature selection setting, and the exact format of our problem is as follows. There is a target parameter vector $\theta^*\in (0,1)^d$. We have a $n \times d$ matrix $X$ ($n$ vectors, each of $d$ features)
and we have $Y=X\times \theta^* + W$ where $W$ itself is a small noise, and $Y$ is the vector of noisy observations about  $\theta^*$, where the number of non-zero elements (denoting the actual relevant features) in $\theta^*$ is bounded by $s$ namely, $|\Supp(\theta^*)|\leq s$. 
The goal of the feature selection is to find a model $\hat{\theta}$, given $\concat{X}{Y}$, such that $\Supp(\hat{\theta})=\Supp(\theta^*).$   

The Lasso Estimator tries to learn $\theta^*$ by optimizing the regularized loss with regularization parameter $\lambda$ and obtain the solution $\hat{\theta}_\lambda$  as
$$\hat{\theta}_\lambda = \argmin_{\theta\in {(0,1)^d}}{\frac{1}{n}\cdot\norm{Y-X\times\theta}_2^2 + \frac{2\lambda}{n}\cdot \norm{\theta}_1}.$$
We use $\Learn(\concat{X}{Y},\lambda)$ to denote $\hat{\theta}_\lambda$, as learned by the Lasso optimization described above. When we $\lambda$ is clear from the context, we use $\Learn(\concat{X}{Y})$ and $\hat{\theta}$.

We also use $\Loss(\htet,\concat{X}{Y}, \lambda)$ (and $\Loss(\htet,\concat{X}{Y}, \lambda)$ when $\lambda$ is clear from the context) to denote the ``scaled up'' value of the Lasso's objective function  
$$\Loss(\htet,\concat{X}{Y})= \norm{Y-X\times\htet}_2^2 +2\cdot\lambda\cdot \norm{\htet}_1.$$

It is known by a work of  Wainwright~\cite{wainwright2009sharp} that under proper conditions Lasso estimator can recover the correct feature vector (See Theorems \ref{thm:LassoRecoveryGen} and \ref{thm:LassoRecovery} in Appendix \ref{sec:gen_lasso} for more details.) The robust version of this result, where part of the training data is chosen by an adversary, is also studied in Thakurta et al. \cite{thakurta2013differentially}. (See Theorems \ref{thm:LassoRobustRecovery} and \ref{thm:LassoRobustRecoveryGen} in Appendix \ref{sec:gen_lasso} for more details.)  However, the robust version considers robustness against data-aware adversaries that can see the dataset and select the poisoning points based on the rest of training data. In the following theorem, we show that the robustness against data-oblivious adversaries could be much higher than robustness against data-aware adversaries.

\parag{Separation for feature selection.}
We prove the existence of a feature selection problem such that, with high probability, it stays secure in the data-oblivious attack model of Definition \ref{security:Feature}, while the same problem's setting is highly vulnerable to poisoning adversaries as defined in the data-aware threat model of Definition \ref{security:Feature}. We use Lasso estimator for proving our separation result.

\newcommand{\Normal}{\mathcal{N}}
\begin{theorem}\label{thm:main}
For any $k\in \N$ and $\eps_1< \eps_2 \in(0,1)$, there exist an $n,d\in\N$, $\sigma\in \R$ and $\theta^*\in \R^d$ such that the distribution $\cD\equiv(\cX,\cY)$ for $\cX\equiv\Normal(0,\sigma^2)^{n\times d}$ and $\cY\equiv X\times \theta^* + \Normal(0,1/4)$ is recoverable using Lasso estimator, meaning that with high probability over the randomness of sampling a dataset  $\concat{X}{Y}\gets \cD^n$ we have 
$$\Supp(\Lasso(\concat{X}{Y}) = \Supp(\theta^*),$$
while the advantage of any  data-oblivious adversary in changing the support set is at most $\eps_1$. Namely for any data-oblivious adversary $A$ we have

$$\Ex_{\substack{\cS\gets D}}\Big[\OblFt(A,k, \cD, \Lasso,n)\Big] \leq \eps_1$$

 On the other hand, there is an adversary that can win the data-aware security game with probability at least $\eps_2$. Namely, there is an data-aware adversary $A$ such that
  $$\Ex_{\substack{\cS\gets D}}\Big[\FullFt(A,k, \cD, \Lasso,n)\Big] \geq \eps_2.$$
\end{theorem}
\parag{The main idea behind the proof.} To prove the separation, we use the fact that data-oblivious adversaries cannot discriminate between the coordinates that are not in the support set of $\theta^*$. Imagine the distribution of data has a property that with high probability there exists a unique feature that is not in the support set, but it is possible to add that feature to the support set with a few number of poisoning examples. We call such a feature an ``unstable'' feature. Suppose the distribution also has an additional property that each coordinate has the same probability of being the unstable feature. Then, the only way that adversary can find the unstable feature is by looking into the dataset. Otherwise, if the adversary is data-oblivious, it does not have any information about the unstable feature and should attack blindly and pick one of the coordinates at random. On the other hand, the data-aware adversary can investigate the dataset and find the unstable feature. In the rest of this section we formalize this idea by constructing a distribution $D$ that has the properties mentioned above. 

Below we first define the notion of stable and unstable features and then formally define two properties for a distribution $\cD$ that if satisfied, we derive Theorem \ref{thm:main} for it.

\begin{definition} [Stable and unstable coordinates]
Consider a dataset $\concat{X}{Y}\in \R^{n\times(d+1)}$ with a unique solution $\hat{\theta}_\lambda$ for the Lasso minimization. $\concat{X}{Y}$ is $k$-unstable on coordinate $i\in[d]$ 

if the $i\th$ coordinate of the feature vector obtained by running Lasso on $\concat{X}{Y}$ is 0, namely $\Learn\left(\concat{X}{Y}\right)_i =0$, and there exist a data set $\concat{X'}{Y'}$ with size $k$ and $\ell_\infty$ norm at most $1$ on each row such that
$i \in \Supp\left(\Learn\left(\concatt{X}{Y}{X'}{Y'}\right)\right).$
On the other hand, $\concat{X}{Y}$ is $k$-stable on a coordinate $i$, if for all  datasets $\concat{X'}{Y'}$ with $k$ rows and  $\ell_\infty$ norm at most 1 on each row we have $$\Sign(\Learn\left(\concat{X}{Y}\right)_i) = \Sign\left(\Learn\left(\concatt{X}{Y}{X'}{Y'}\right)_i\right).$$
\end{definition}

The following definitions capture two properties of a distribution $D$. The first property states that with high probability over the randomness of $D$, a dataset sampled from $D$ has at least one unstable feature. 

 \begin{definition}[$(k,\delta)$-unstable distributions]\label{def:stable}
 A distribution $D$ is $(k,\eps_2)$-unstable if it is $k$-unstable on at least one coordinate with probability at least $(\eps_2)$. Namely
 $$\Pr_{S\gets D}[\Exists i \in [d]: \text{ The $i\th$ feature is $k$-unstable on $S$}]\geq \eps_2.$$
 \end{definition}
 The following notion defines the resilience of a distribution against a single poison dataset. In a nutshell, a distribution is resilient if there does not exist a universal poisoning set that can be effective against all the datasets coming from that distribution.
\begin{definition} \label{def:resilient}[$(k,\eps)$-resilience]
 A distribution $D$ over $\R^{n\times(d+1)}$ is $(k,\eps)$-resilient if 
for any poisoning dataset $\cS'$ of size $k$ and $\ell_\infty$ norm bounded by $1$ we have
$$\Pr_{\cS\gets D}[\Supp\left(\Lasso\left(\concatv{\cS}{\cS'}\right)\right)\neq\Supp(\Lasso(S))]\leq \eps.$$

\end{definition}
\begin{remark}
Note that Definitions \ref{def:stable} and \ref{def:resilient} have an implicit dependence on $n$, the size of the dataset sampled from the distribution that we omit from the notation for simplicity.
\end{remark}
Before constructing a distribution $D$ we first prove the following Proposition about $(k,\delta)$-unstable and $(k,\eps)$-resilient distributions. The proof   can be found in Appendix \ref{proof:prop}
\begin{proposition}[Separation for unstable yet resilient distributions] \label{prop:main}If a data distribution is $(k,\eps_1)$-resilient and $(k,\eps_2)$-unstable, then there is an adversary that wins the data-aware game of definition \ref{security:Feature} with probability $\eps_2$, while no adversary can win the data-oblivious game with probability more than $\eps_1$.
\end{proposition}
\vspace{-14pt}
\subsection{(In)Stability and Resilience of Gaussian}\label{subsec:construction}

The only thing that remains to prove Theorem \ref{thm:main} is to show that Gaussian distributions with proper parameters are $(k,\eps_2)$-unstable and $(k,\eps_1)$-resilient at the same time.  Here we sketch the two steps we take to prove this.

\parag{Gaussian is Unstable.}  We first show that each feature in the Gaussian sampling process has a probability of being k-unstable that is proportional to $e^{\lambda-k}$. Note that the unstability of $i$-th feature is independent from all other features and also note that the probability is independent of $d$. This shows that, if $d$ is chosen large enough, with high probability, there will be at least one coordinate that is $k$-unstable. However, note that the probability of a particular feature being unstable is still low and we are only leveraging the large dimensionality to increase the chance of having an unstable feature. Roughly, if we select $d=\omega(\eps_2/\eps_1)$, we can make sure that the ratio of the success rate between data-aware and data-oblivious adversary is what we need. The only thing that remains is to select $n,\lambda$ and $\sigma$ in a way that the data oblivious adversary has success rate of at most $\eps_1$ and at least $\Omega(\eps_1)$. 

This result actually shows the tightness of the robustness theorem in \citep{thakurta2013differentially} (See Theorem \ref{thm:LassoRobustRecoveryGen} for the full description of this result). The authors in \citep{thakurta2013differentially} show that running Lasso on Gaussian distribution can recover the correct support set, and is even robust to a certain number of adversarial entries. Our result complements theirs and shows that their theorem is indeed tight. Note that the robustness result of \citep{thakurta2013differentially} is against dataset-aware attacks. In the next step, we show a stronger robustness guarantee for data-oblivious attacks in order to prove our separation result. See Appendix \ref{proof:prop} for a formalization of this argument.

\parag{Gaussian is Resilient.}
We show the LASSO is resilient when applied on Gaussian of any dimension. In particular, we show that if the adversary aims at adding a feature to the support set of the model, it should ``invest'' in that feature meaning that the $l_2$ weight on that feature should be high across all the poison entries. The bound on the $l_2$ norm of each entry will prevent the adversary to invest on all features and  therefore, the adversary has to predict which features will be unstable and invest in them. On the other hand, since Gaussian is symmetric, each feature has the same probability of being unstable and  the adversary will have a small chance of succeeding. In a nutshell, by selecting $\lambda=\Omega(k+\sigma\sqrt{(n+k)\ln(1/\eps_1)})$ we can make sure that the success probability of the oblivious adversary is bounded by $\eps_1$. This argument is formalized in Appendix~\ref{proof:prop}. 

%% file: experiments.tex
\subsection{Experiments}\label{sec:experiments}
In this section, we highlight our experimental findings on both synthetic and real data to compare the power of data-oblivious and data-aware poisoning attacks in the context of feature selection. Our  experiments  empirically support our separation result in Theorem \ref{thm:main}. 

\paragraph{Attacking a specific feature:} For our experiments on feature selection we use the following attack similar to the one described in Section \ref{sec:optimal-attack}. 
To attack a feature $i$ with $k$ examples, we use a dataset $\cS'=\concat{X'}{Y
 '}$ as follows:
\begin{align*}
    X' = \left[
    \begin{array}{cccc}
         0 & \dots & 1 & 0\\
         \vdots&\ddots &\vdots &\vdots\\
         0 & \dots & 1 & 0
    \end{array}
    \right] \in \mathbb{R}^{k \times d} 
    , 
    Y' = \left[
    \begin{array}{c}
          1\\
         \vdots\\
         1
    \end{array}
    \right] \in \mathbb{R}^{k \times 1}.
\end{align*}
The attack then adds $S'$ to the training set. Note that this attack is oblivious as it does not use the knowledge of the clean training set.



\paragraph{How to select the feature to attack?}
In total, the attack aims to find the feature that requires the minimum number of rows to add to $\Supp(\hat{\theta}),$ where $\hat{\theta}$ is the learned parameter vector from Lasso. From Section \ref{sec:optimal-attack} we know that the attack above is almost optimal for adding a specific feature. However, the attack still has to decide which feature to attack as some features are much more unstable than others. A question that might arise here is why  the power of data-oblivious and data-aware adversaries would change if the same underlying attack is used by both of them. The key here is that the data-aware adversary can search for the most vulnerable feature and target that feature specifically. However, the data-oblivious adversary has to choose the features that will \emph{probably} be unstable. Namely, the attack should sample multiple dataset from the distribution and identify which features get unstable more frequently. However, there could be a lot of entropy in the instability of different features which makes the job of oblivious adversary hard. 

We provide experiments on Gaussian synthetic data and experiments with real datasets. In both sets of experiments, we simulate the power of an adversary that is completely oblivious to the dataset all the way up to a full-information adversary. We change the power of partial information adversaries that only have $p\%$ information about the dataset at regular intervals in between $0\%$ and $100\%.$

\paragraph{Our partial-knowledge attack:} The attack first explores through the part of data that it has access to and identifies which feature is the most unstable feature.
The key here is that the data-aware adversary can search for the most vulnerable feature in the available data. Then, the attack will use that feature to craft poison points that create maximum correlation between that feature and the response variable. See Appendix \ref{sec:optimal-attack} for more details.

\paragraph{Experiments with Gaussian distribution.} For the synthetic experiment, we demonstrate the separation result occurs for a large dataset sampled from a Gaussian distribution. For $n = 300$ rows and $d = 5 \times 10^5$ features, we demonstrate that unstable features  occur for a dataset drawn from $\mathcal{N}(0, 1)^{n \times d}.$ For the LASSO algorithm, we use the hyperparameter of $\lambda = 2\sigma\sqrt{n\log p}$. We vary the ``knowledge'' the adversary has of the dataset from $p = 0, 5, 10, \dots 95, 100 \%$ by only showing the adversary a random sample of $p\%$ (for $p = 0$, the adversary is completely oblivious and so must choose a feature uniformly at random). The adversary then chooses the most unstable feature out of their $p\%$ of the data and perform the attack on that feature to add it to the $\Supp(\hat{\theta}).$ We observe a clear separation between data-oblivious, data-aware, and partially-aware adversaries in Figure \ref{fig:synthetic_exp}.

\begin{figure}[ht]
     
\centering
    \includegraphics[scale=0.1]{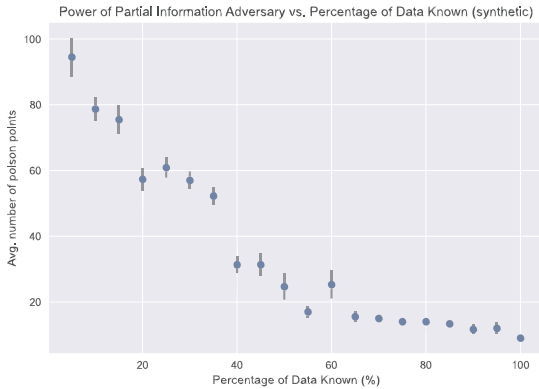}
     \caption{\textit{Synthetic experiment.}\label{fig:synthetic_exp} 
     The y-axis is the average (over 30 random $p\%$ splits of the dataset given to the adversary) number of poison points needed to add the feature to $\hat{\theta}$.The leftmost point shows the power of an oblivious adversary while the rightmost point shows the power of a full-information adversary. The oblivious adversary needs significantly more poison points, on average, to add their uniformly chosen feature to $\Supp(\hat{\theta})$.\vspace{-5pt}\\}
     
    
\end{figure}


\vspace{-20pt}
\paragraph{Experiments with real data.} We also consider MNIST and four other datasets  used widely in the feature selection literature to explore this separation in real world data: Boston, TOX, Protate\_GE, and SMK. \footnote{TOX, SMK, and Prostate\_GE can be found here: \url{http://featureselection.asu.edu/datasets.php}. Boston can be found with scikit-learn's built-in datasets:\\ \url{https://scikit-learn.org/stable/modules/generated/sklearn.datasets.load_boston.html}}

\begin{itemize}
    \item \textbf{Boston.} \cite{boston} (506 examples, 13 features) The task in this dataset is to predict the median value of a house in the Boston, Mass. area, given attributes that describe its location, features, and surrounding area. The outcome variable is continuous in the range $[0, 50].$
    \item \textbf{TOX.} \cite{tox} (171 examples, 5,748 features) The task in this dataset is to predict whether a patient is a myocarditis and dilated cardiomyopathy (DCM) infected male, a DCM infected female, an uninfected male, or an uninfected female. Each feature is a gene, and each example is a patient. The outcome variable is discrete in $\{1, 2, 3, 4\},$ for each of the four possibilities.
    \item \textbf{Prostate\_GE.} \cite{prostate} (102 examples, 5,966 features) The task in this dataset is to predict whether a patient has prostate cancer. Each feature is a gene, and each example is a patient. The outcome variable is binary in $\{0, 1\},$ for cancer or no cancer.
    \item \textbf{SMK.} \cite{tox} (187 examples, 19,993 features) The task in this dataset is to predict whether a smoker has lung cancer or not. Each example is a smoker, and each faeture is a gene. The outcome variable is binary in $\{0, 1\}$ for cancer or no cancer.
    \item \textbf{MNIST} We use MNIST data for feature selection. Each pixel number constitutes a feature and we want to recover a subsst of them that are more relevant.
\end{itemize}


We first preprocess the data by standardizing to zero mean and unit variance. Then, we chose $\lambda$ such that the resulting parameter vector $\hat{\theta}$ has a reasonable support size (at least 10 features in the support); this was done by searching over the space of $\lambda/n \in [0, 1.0],$ and resulted in $\lambda = 50.1$ for Boston, $\lambda = 9.35$ for SMK, $\lambda = 17$ for TOX, $\lambda = 5.1$ for Prostate, and $\lambda = 1000$ for MNIST. Just as in the synthetic experiments, we allow the adversary to have the knowledge of $p = 0, 5, 10, \dots , 95, 100\%$ fraction of the data. Denote the features \textit{not} in $\Supp(\hat{\theta})$ as $\mathcal{G}.$ We attack each feature $i \in \mathcal{G}$ with the same attack as our synthetic experiment, where $X' \in \R^{k \times d}$ and $Y' \in \R^{k \times 1}.$ We plot the average best value of $k$ needed by the adversary to add a feature to $\Supp(\hat{\theta})$ against how much knowledge ($p\%$) of the dataset they have. We show the results for SMK and TOX in Figure \ref{fig:SMK-TOX} and the result for MNIST in Figure \ref{fig:MNIST}. 

\begin{figure}[ht]\label{fig:realworld_exp}
        \begin{subfigure}[b]{0.48\linewidth}

    \includegraphics[scale=0.1]{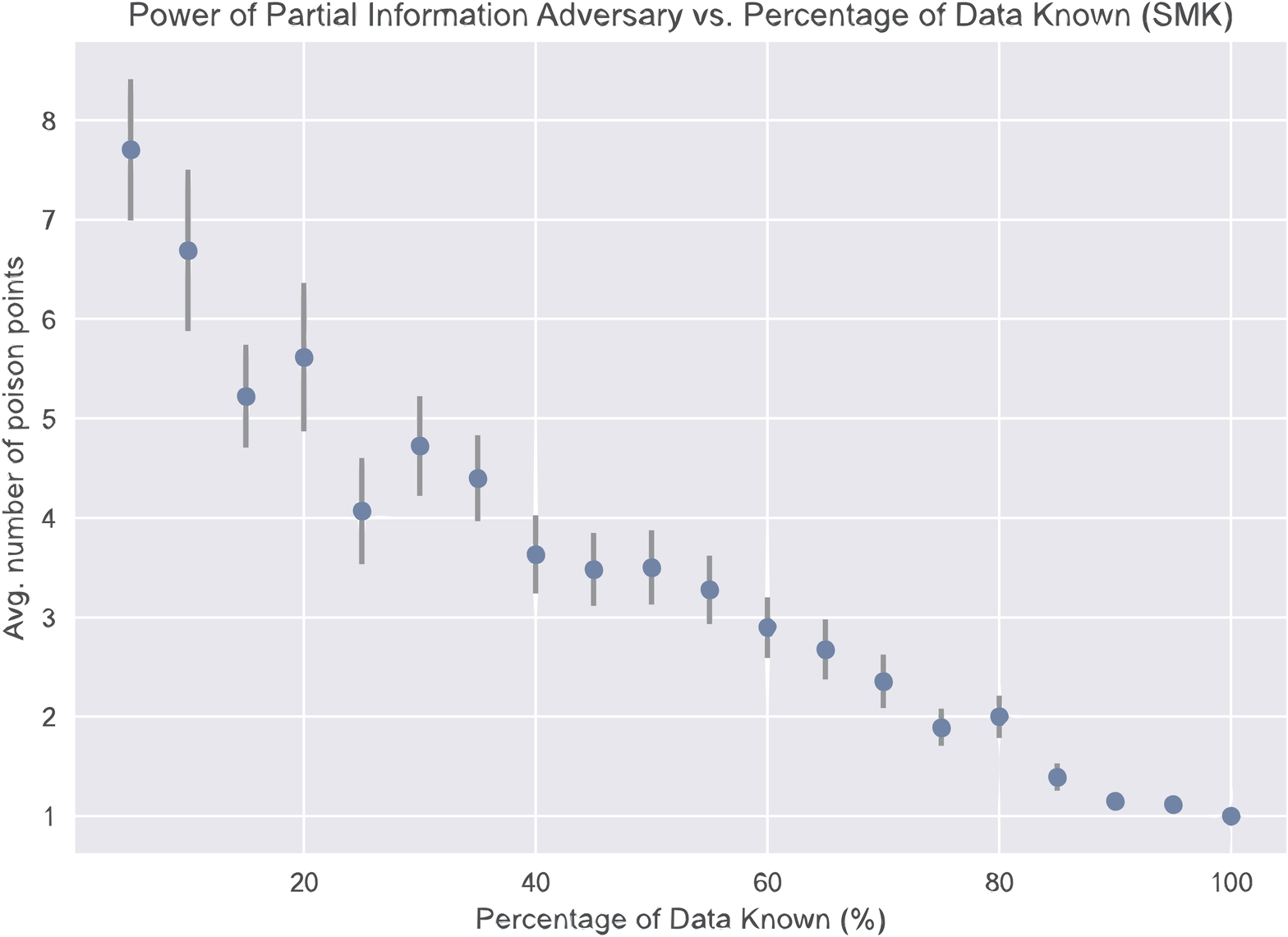}
\end{subfigure}
    \begin{subfigure}[b]{0.48\linewidth}
    \includegraphics[scale=0.1]{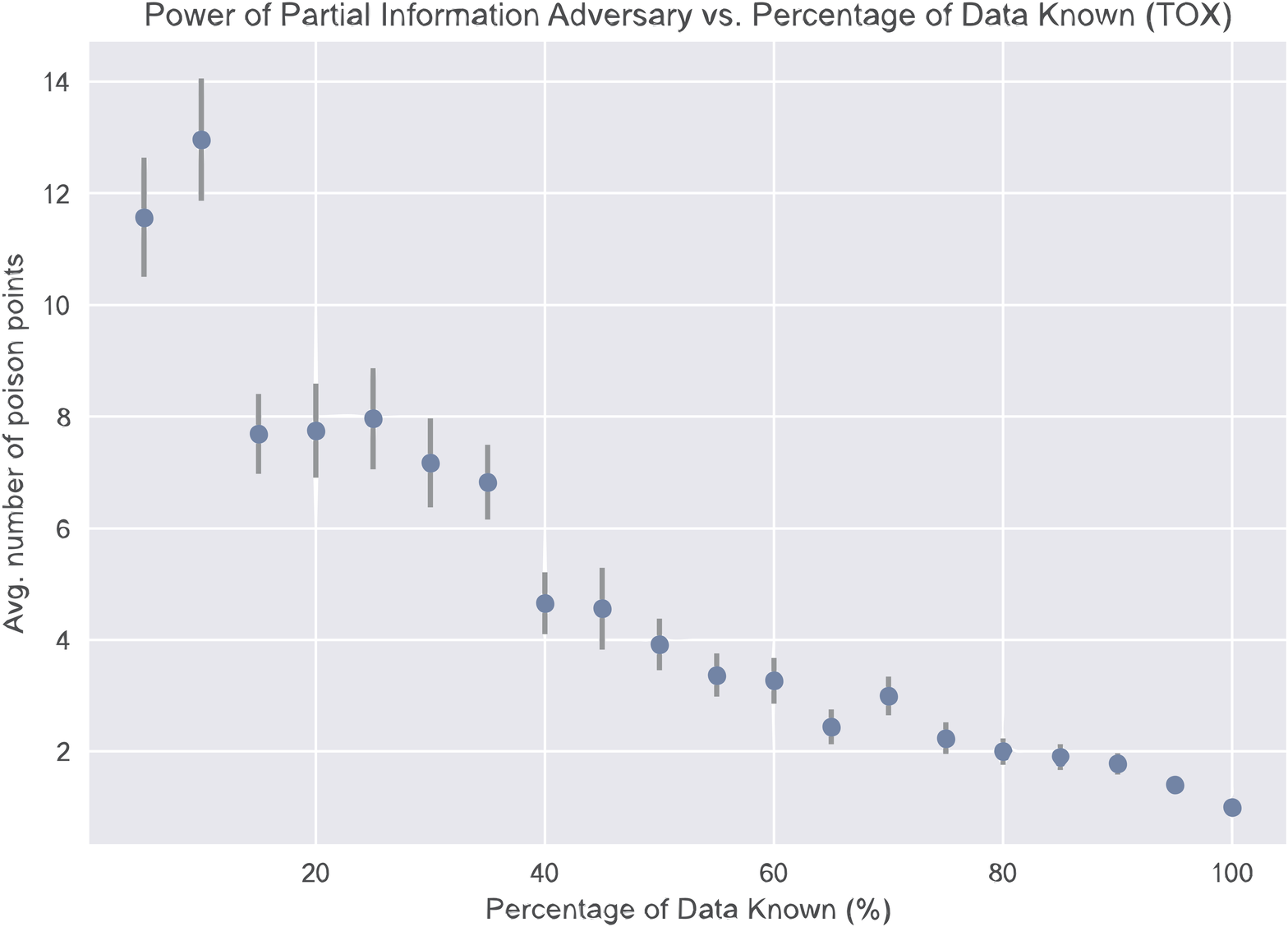}
\end{subfigure}
     \caption{\textit{SMK and TOX Experiments.}\label{fig:SMK-TOX} 
     The behavior of attack on these two datasets is very similar to synthetic experiments. We believe this is because of the noisy nature of these feature selection datasets which causes them to be similar to the Gaussian distribution. Since the noise is large, even given the half of the dataset, the attacker cannot identify the most unstable feature.}
     \begin{subfigure}[b]{0.98\linewidth}
     \centering
     \includegraphics[scale=0.1]{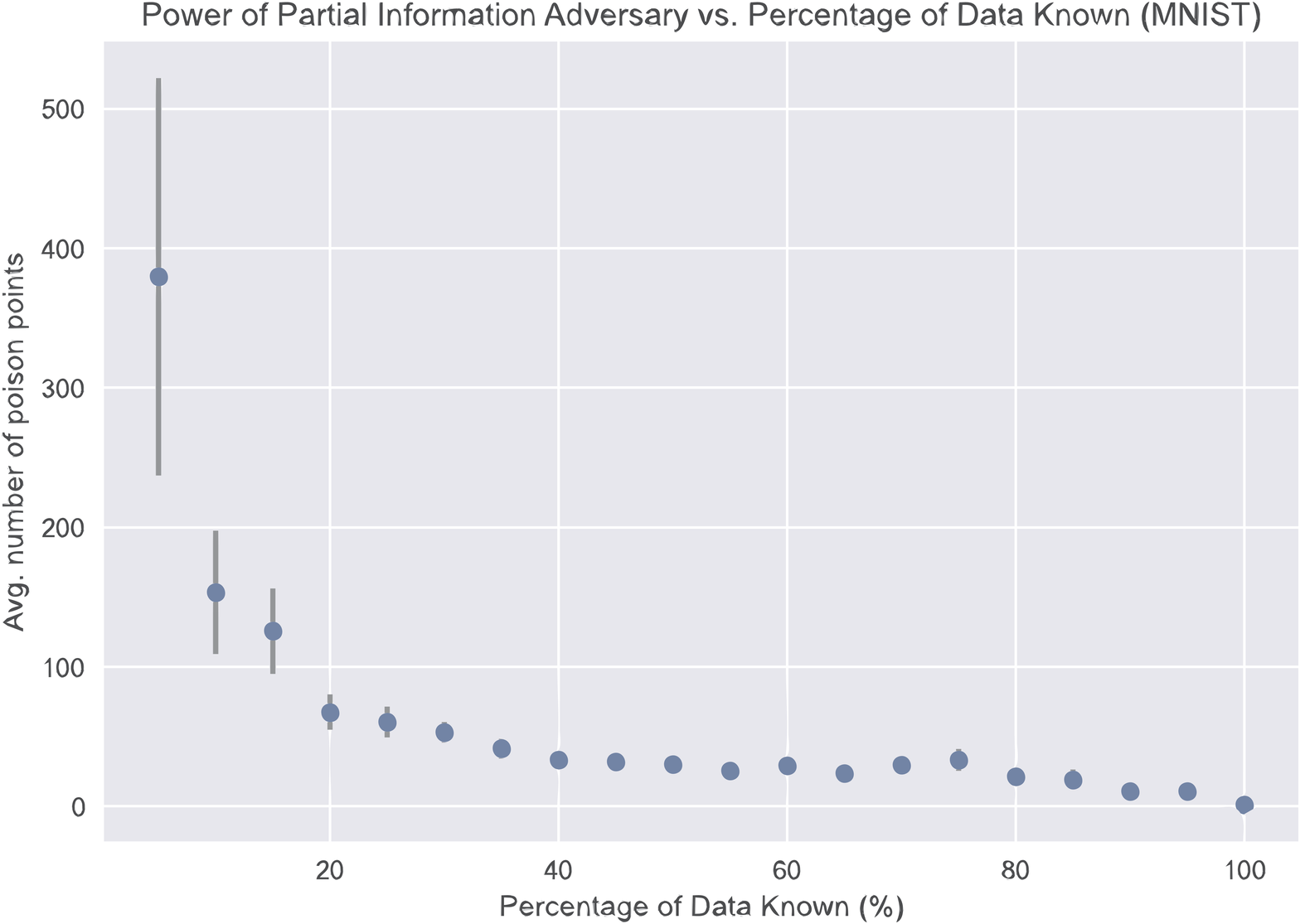}
     \end{subfigure}
     \caption{\textit{MNIST experiments.}\label{fig:MNIST} Compared to other experiments, the number of poison points drops faster as the percentage of data-awareness grows. This can be explained by separability (less noisy nature) of MNIST dataset.\vspace{-10pt}}
\end{figure}

%% file: Conclusion_neurips.tex
\section{Conclusion} \label{sec:conc}
\vspace{-10pt}
In this paper we initiated a formal study of the power of data-oblivious adversaries who do not have the knowledge of the training set in comparison with data-aware adversaries who know the training data completely before adding poison points to it. Our main result proved a separation between the two threat models  by constructing a sparse linear regression problem. We show that in this natural problem, Lasso estimator is robust against data-oblivious adversaries that aim to add a non-relevant features to the model with a certain poisoning budget. On the other hand, for the same problem, we prove that data-aware adversaries, with the same budget,  can find specific poisoning examples based on the rest of the training data in such a way that they can successfully add non-relevant features to the model. We also experimentally explored the partial-information adversaries who only observe a fraction of the training set and showed that even in this setting, the adversary could be much weaker than full-information adversary. As a result, our work sheds light on an important and yet
subtle aspect of modeling the threat posed by poisoning
adversaries. We, leave open the question of separating different aspects of poisoning threat model including computational power of adversaries, computational power of learners, clean-label nature of adversaries and etc. 

%% file: Acknolegments.tex
\paragraph{Acknowledgments.}
Mohammad Mahmoody was supported by NSF grants CCF-1910681 and CNS-1936799. Sanjam Garg is supported in part by DARPA under Agreement No. HR00112020026, AFOSR Award FA9550-19-1-0200, NSF CNS Award 1936826, and research grants by the Sloan Foundation, and Visa Inc. The work is partially supported by Air Force Grant FA9550-18-1-0166, the
National Science Foundation (NSF) Grants CCF-FMitF-1836978, IIS-2008559,
SaTC-Frontiers-1804648 and CCF-1652140, and ARO grant number W911NF17-1-0405. Somesh Jha is partially supported by the DARPA GARD problem under agreement number 885000. Any opinions, findings and conclusions or recommendations expressed in this material are those of the author(s) and do not necessarily reflect the views of the United States Government or DARPA.

%% file: seperation_for_population_risk.tex
\section{Separating Data-oblivious and Data-aware Poisoning for Classification}\label{sec:sepRiskproof}
In this section, we show a separation on the power of data-oblivious and data-aware poisoning attacks on classification. In particular we show that empirical risk minimization (ERM) algorithm could be much more susceptible to data-aware poisoning adversaries, compared to data-oblivious adversaries. 

Before stating our results, we shall clarify that the attack on classification can also focus on different goals. One goal could be to increase the population risk of the resulting model $\theta'$ that the learner generates from the (poisoned) data $\cS'$, compared to the model $\theta$ that  would have been learned from $\cS$ \citep{steinhardt2017certified}. A different goal could be to make $\theta'$ fail on a particular test set of adversary's interest, making it a \emph{targeted poisoning}~\citep{barreno2006can,shen2016uror} or increase the probability of a general ``bad predicate'' of $\theta$~\citep{mahloujifar2018curse}. Our focus here is on attacks that aim to increase the population risk.

We begin by giving a formal definition of the threat model.

\begin{definition}[data-oblivious and data-aware data injection poisoning for population risk]\label{security:PopRisk}
We first describe the \emph{data data-oblivious} security game  between a challenger $C$ and an adversary $A$, and then will describe how to modify it into a data-aware variant. Such game is parameterized by adversary's budget $k$, a data set $\cS$ a learning algorithm $L$, and a distribution $D$ over $\cX \times \cY$ (where $\cX$ is the space of inputs and $\cY$ is the space of outputs).\footnote{Since we deal with risk, we need to add $D$ as a new parameter compared to the  games of Definition \ref{security:Feature}.}

\noindent~~~~~$\OblRisk(k, \cS, L, D)$.
\begin{compactenum}

    \item Adversary $A$ generates $k$ new examples $(e'_1,\dots,e'_k)$ and them to $C$.
    \item $C$ produces the new data set $\cS'$ by adding the injected examples to $\cS$.
    \item $C$ runs $L$ over $\cS'$ to obtain (poisoned) model $\theta' \gets L(\cS')$.
    \item $A$'s advantage (in winning the game) will be $\Risk(\theta',D) = \Pr_{(x,y) \gets D}[\theta'(x)\neq y]$.
    \footnote{Note that this is a real number, and more generally we can use any loss function, which allows covering the case ore regression as well.}
\end{compactenum}

In the \emph{data-aware} security game, all the steps are the same as above, except that in the first step the following is done.

\noindent~~~~~$\FullRisk(k, \cS,L,D)$.
\begin{compactitem}
\item Step 0:   {\color{purple}  $C$ sends $\cS$ to $A$}.
\item The rest of the steps are the same as those of the game $\OblRisk(k, \cS,L,D)$.
\end{compactitem}
\end{definition}

One can also envision variations of Definition \ref{security:PopRisk} in which the goal of the attacker is to increase the error on a particular instance (i.e., a \emph{targeted} poisoning \cite{barreno2006can,shen2016uror}) or use other poisoning methods that eliminate or substitute poison data rather than just adding some.

We now state and prove our separation on the power of data-oblivious and data-aware poisoning attacks on classification. In particular we show that empirical risk minimization (ERM) algorithm could be much more susceptible to data-aware poisoning adversaries, compared to data-oblivious adversaries.

\begin{theorem}\label{thm:SepRisk}
There is a distribution of distributions $\distdist$ 

 such that there is a data injecting adversary with budget $\eps\cdot n$ that wins the data-aware security game for classification by advantage $\eps$, namely
$$\exists A : \Ex_{\substack{D\gets \distdist\\\cS \gets D^n}}\Big[\text{\em Advantage of } A \text{ \em in } \FullRisk(\eps\cdot n, \cS, \ERM, D)\big)\Big] \geq \Omega(\eps).$$
On the other hand, any adversary will have much smaller advantage in the data-oblivious game. Namely, the following holds.
$$\forall A : \Ex_{\substack{D\gets \distdist\\\cS \gets D^n}}\Big[\text{\em Advantage of } A \text{ \em in } \OblRisk(\eps\cdot n,\cS, \ERM, D)\big)\Big] \leq O(\eps^2).$$
\end{theorem}


\begin{proof}[Proof]
Here we only sketch the proof. To prove this we use the problem of learning concentric halfspaces in Gaussian space $\cN(0,1)^2$. We assume that the prior distribution is uniform over all concentric halfspaces. We first show that there is a data-aware attack with success $(\eps)$. The way this attack works is as follows, attacker first uses $\ERM$ to learn a halfspace $w_1$ on the clean data. Assume this halfspace has risk $\delta$. Then the attacker selects another halfspace $w_2$ that disagrees with $w_1$ on $\eps\cdot n-1$ number of points in the training data. Note that this is possible because the attacker can keep rotating the half-space until it has exactly $n\cdot \eps-1$ points disagreeing with $w_1$. Now if the adversary puts all the poison points on the separating line for $w_1$ and with the opposite label of what $w_1$ predicts, then $\ERM$ would prefer $w_2$ over $w_1$. Therefore the empirical error of $w_2$ on clean dataset would be at least equal to $\eps-\delta$. Now if we increase $n$, the generalization error would go to zero which means the population error of $w_2$ would be close to $\eps-\delta$. Also, since we are assuming the problem is realizable by half-spaces, it means $\delta$ would also converge to $0$. Therefore, the final population risk could be bounded to be at least $\eps/2$ for $n$ larger than some reasonable values. Which means our proof for the data-aware attack is complete. 

Now, we show that no data-oblivious adversary cannot increase the error by more than $\eps^2$, on average. The reason behind this boils down to the fact that each poison point added can affect at most $\epsilon$-fraction of the choices of ground truth. To be more specific, we can fix the poison data to a fixed set $D_p$ with size $\epsilon\cdot n$, as we can assume that the data-oblivious adversary is deterministic. Now if we fix the ground truth to some $w^g$, and define the epsilon neighborhood of a model  $w$ to be all the points that have angle at most $\epsilon\cdot\pi$ with $w$ and denote it by $w_\epsilon$. Then we have
\begin{align}\Ex_{\substack{X_c \gets \cN(0,1)^n\\ y_c=w^g(X_c)\\
D_c=(X_c,y_c)\\
w^p=\ERM(D_c\cup D_p), w^c=\ERM(D_c)}}[\Risk(w^p) - \Risk(w^c)]
&\leq \Ex_{\substack{X_c \gets \cN(0,1)^n\\ y_c=w^g(X_c)\\
D_c=(X_c,y_c)\\
w^p=\ERM(D_c\cup D_p)}}[\Risk(\ERM(w^p))]\nonumber
\\&\leq\Ex_{\substack{X_c \gets \cN(0,1)^n\\ y_c=w^g(X_c)\\
D_c=(X_c,y_c)\\
w^p=\ERM(D_c\cup D_p)}}[\Risk_{D_c}(w^p)] + \delta\numberthis\
\end{align}
where $\delta$ is the generalization parameter that relates to $n$ and goes to $0$ with rate $1/n$.
Now consider an event $E$ where the angle between $w^c$ and $w^g$ is at most $\epsilon\cdot\pi$ and $w^g_{2\epsilon} \cap X_c$  has at least $\epsilon$ points on each side of $w^g$. We denote the probability of this event by $1-\delta'$ and we know that $\delta'$ goes down to $0$ as $n$ grows, by rate $1/\sqrt{n}$ (Using Chernoff Bound). Now we can observe that conditioned on $E$, we have $\Risk_{D_c}(w_p)\leq |w^g_{2\epsilon} \cap X_c|$. This is because the poison points cannot increase the errorn by more than $\epsilon$ so $w^p$ would disagree with $w^c$ on at most $\epsilon\cdot n$ points in $D_c$. On the other hand, we know that in $2\epsilon$ neighborhood of $w_g$ there are at least $\epsilon\cdot n$ points on each side of $w_g$, which means there are at least $\epsilon\cdot n$ points on each side of $w^c$ (because $w^c$ and $w^g$ would fall between the same two points in $D_c$). Therefore, the poisoned model, would definitely be in the $2\cdot\epsilon$ neighborhood of the $w_g$. At the same time, we know that the maximum number of points in $D_c$ that $w^g$ and $w^p$ disagree on are at most equal to the number of poison points that fall in their disagreement region. And since the disagreement region is a subset of $w^g_{2\epsilon}$, we have the maximum number of points in $D_c$ that $w^g$ and $w^p$ disagree on are at most equal to $|w^g_{2\epsilon} \cap X_c|$. Now having this, using Equation (12) we can write
\begin{align*}\Ex_{\substack{X_c \gets \cN(0,1)^n\\ y_c=w^g(X_c)\\
D_c=(X_c,y_c)\\
w^p=\ERM(D_c\cup D_p), w^c=\ERM(D_c)}}[\Risk(w^p) - \Risk(w^c)]
&\leq \frac{|D_p \cap w^g_{2\epsilon}|}{n} + \delta + \delta'
\end{align*}
Now by also taking the average over $w^g$ we get
\begin{align*}\Ex_{\substack{w^g\gets\distdist\\X_c \gets \cN(0,1)^n\\ y_c=w^g(X_c)\\
D_c=(X_c,y_c)\\
w^p=\ERM(D_c\cup D_p), w^c=\ERM(D_c)}}[\Risk(w^p) - \Risk(w^c)]
&\leq \Ex_{w^g\gets \distdist}[\frac{|D_p \cap [w^g_{2\epsilon}|}{n}] + \delta + \delta'= 2\epsilon^2 + \delta + \delta'
\end{align*}
 As $\delta$ and $\delta'$ converge to 0 with rate $1/\sqrt{n}$, for $n\geq \omega(1/\epsilon^2)$ we have 
 $$\Ex_{\substack{w^g\gets\distdist\\X_c \gets \cN(0,1)^n\\ y_c=w^g(X_c)\\
D_c=(X_c,y_c)\\
w^p=\ERM(D_c\cup D_p), w^c=\ERM(D_c)}}[\Risk(w^p) - \Risk(w^c)]
\leq O(\epsilon^2).$$
\end{proof}
We also state the theorem about separation of data-oblivious and data-aware adversaries in the data elimination setting. This theorem has shows that the gap between data-oblivious and data-aware adversaries could be wider in the data elimination settings. We use $\FullRisk^{\elim}$ and $\OblRisk_{\elim}$ to denote the information risk in presence of data-oblivious and data-aware data elimination attacks.
\begin{theorem}
There is a distribution of distributions $\distdist$ 

 such that there is a data elimination adversary with budget $\eps\cdot n$ that wins the data-aware security game for classification by advantage $\eps$, namely
$$\exists A : \Ex_{\substack{D\gets \distdist\\\cS \gets D^n}}\Big[\text{\em Advantage of } A \text{ \em in } \FullRisk^{\elim}(\eps\cdot n, \cS, \ERM, D)\big)\Big] \geq \Omega(\eps).$$
On the other hand, any adversary will have much smaller advantage in the data-oblivious game. Namely, the following holds. 
$$\forall A : \Ex_{\substack{D\gets \distdist\\\cS \gets D^n}}\Big[\text{\em Advantage of } A \text{ \em in } \OblRisk_{\elim}(\eps\cdot n,\cS, \ERM, D)\big)\Big] \leq e^{-\omega((1-\eps)n)}.$$
\end{theorem}

\begin{proof}
For the negative part on the power of data-aware attacks, we observe that for a fixed $w_g$ the attacker can find a half-space $w_c$ that has angle $\pi\epsilon/2$ with the ground-truth $w_g$, and remove all the points where $w_c$ and $w_g$ disagree. Note that the number of points in the disagreement region would be at most $\epsilon$ with some large probability $1-\delta$ where $\delta$ goes to $0$ with rate $1/\sqrt{n}$. After the adversary removes all the points in disagreement region, the learner cannot distinguish them and will incur an error $\epsilon/2$ on average. We note that this attack is similar to the hybrid attack described in the work of Diochnos et al. \cite{diochnos2019lower}. For the positive result, we make a simple observation that data-oblivious poisoning adversary can only reduce the sample complexity for the learner. In other words, non-removed examples would remain i.i.d examples.  This means that after removal, we can still use uniform convergence theorem to bound the error of resulting classifier.  Since the error of learning realizable half-spcaces will go to zero with rate $\Omega(1/n)$, therefore the average error after the attack would be $\Omega(1/(1-\epsilon)n)).$
\end{proof}


%% file: experiments_details_PR.tex
\subsection{Experiments}\label{ref:class_exp}
In this section, we design an experiment to empirically validate the claim made in Theorem \ref{thm:SepRisk}, that there is a separation between oblivious and data-aware poisoning adversaries for classification. We setup the experiment just as in the proof of Theorem \ref{thm:SepRisk}, as follows.

\par

First, we sample training points $X = x_1, x_2, \dots x_m$ for $m = 1,000$ from the Gaussian space $\mathcal{N}(0, 1)^2$, and pick a random ground-truth halfspace $w^*$ from $\mathcal{N}(0, 1)^2.$ Using $w^*,$ we find our labels $y_1, y_2, \dots y_m$ by taking $(w^*)^T x_k$ for $k \in [m].$ This ensures the data is linearly separable by the homogeneous halfspace produced by $w^*.$ 

\par

To attack this dataset simulating our data-aware adversary with budget $\epsilon$, we construct $\epsilon \cdot m$ poison points $d$ as follows:

\[ d = \cos(\epsilon \pi) \cdot \frac{v}{\|v\|} + \sin(\epsilon \pi) \cdot \frac{w}{\|w\|}, \quad \text{ where } v = \begin{bmatrix} 1, & -\frac{w_1}{w_2}\end{bmatrix}\]

and we add $\epsilon \cdot m$ of these $d$ rows to our dataset. Note that this specific $d$ corresponds to halfspace $w_2$ in our Proof of Theorem \ref{thm:SepRisk}, the halfspace obtained by rotating the original halfspace until it has exactly $ \epsilon \cdot m$ points disagreeing with $w^*$. We label each of these $d$ rows to be $y_d = -(w^*)^T d,$ the opposite label from ground-truth. Then, we train our halfspace via ERM on this poisoned dataset of $m \cdot (1 + \epsilon)$ points (from appending $\epsilon \cdot m$ rows of $d$). We evaluate our poisoned halfspace on another $X' = x_1', x_2', \dots x_m'$ test points from the same Gaussian $\mathcal{N}(0, 1)^n$ distribution.

\par 

To attack this dataset simulating the oblivious adversary, we try three oblivious strategies of attack that an adversary with no knowledge of the dataset might wage, each with $\epsilon$ budget:

\begin{enumerate}
    \item Sample a single random point $p$ from $\mathcal{N}(0, 1)^n$ and repeat it $\epsilon \cdot m$ times. Choose the label $p_y$ uniformly at random from $\{-1, 1\}$.  Poison by adding these $\epsilon \cdot m$ rows to the dataset.
    \item Sample $\epsilon \cdot m$ points IID from $\mathcal{N}(0, 1)^n$ and choose the label $p_y$ uniformly at random from $\{-1, 1\}.$ Label all of the $\epsilon \cdot m$ points with $p_y$. Poison by adding these $\epsilon \cdot m$ rows to the dataset.
    \item Sample $\epsilon \cdot m$ points IID from $\mathcal{N}(0, 1)^n$ and choose the label $p_y$ uniformly at random from $\{-1, 1\}$ for \textit{each point.} That is, we flip a coin to label each poison example, rather than just choosing one label, as in 2. Poison by adding these $\epsilon \cdot m$ rows to the dataset.
\end{enumerate}
We also use the same ERM algorithm, as in the data-aware case, to train the poisoned classifiers on these three oblivious poisoning strategies. 
 \begin{figure}[ht]
    \centering
    \includegraphics[width=0.5\textwidth]{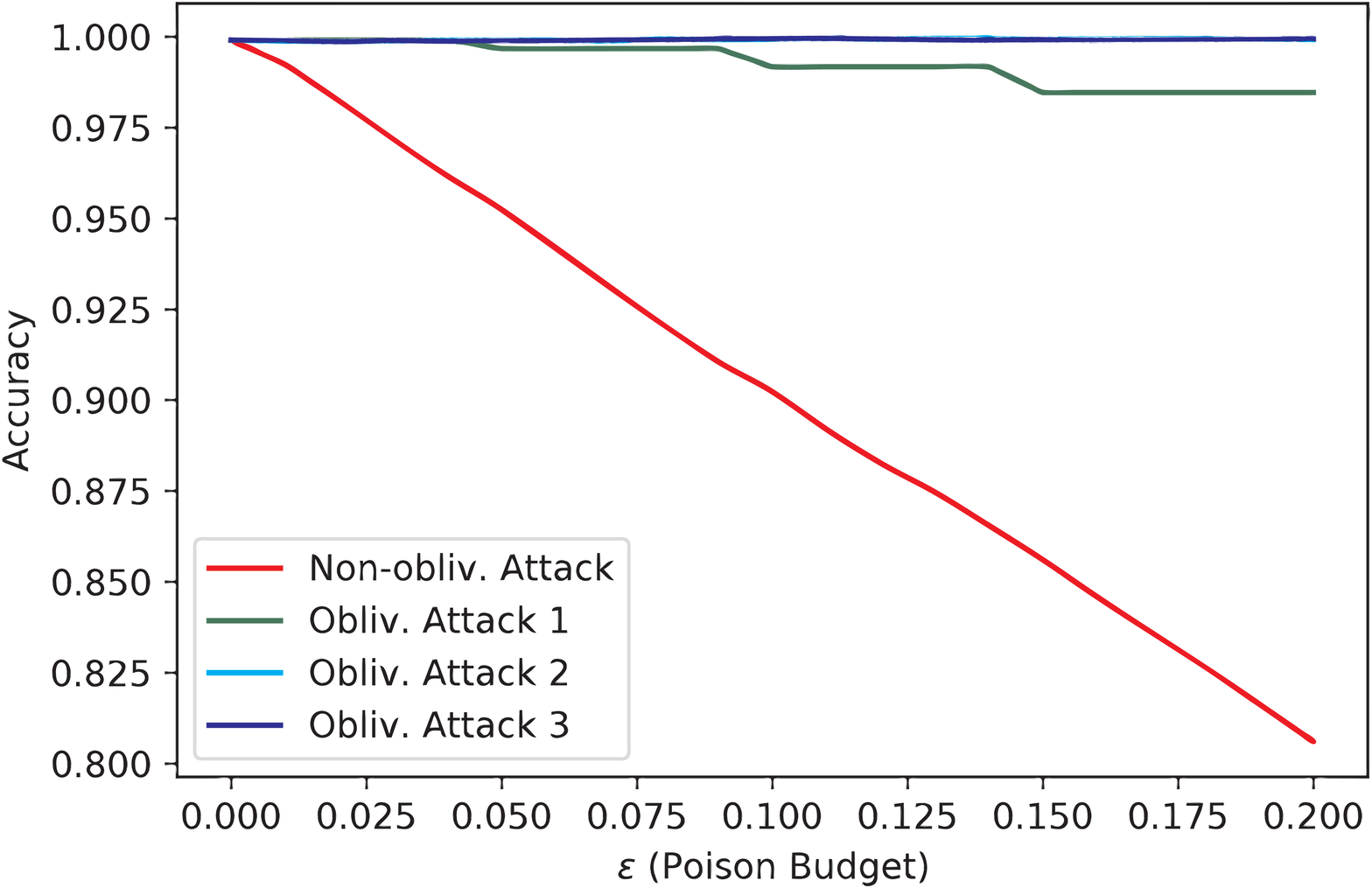}
     \caption{\textit{Oblivious and data-aware poisoning separation in classification.} Over 20 trials, we vary the poisoning budget $\epsilon$ and construct poisoned datasets as discussed above for each adversary. We plot the effect of each adversary's attack on the accuracy of our resulting poisoned ERM halfspace.} 
     \label{fig:classification_exp}
\end{figure}
We repeat this experiment 20 times for poison budget $\epsilon \in \{0, 0.01, 0.02, \dots 0.19, 0.2\}.$ We observe in Figure \ref{fig:classification_exp} that there indeed exists a separation between the power of our data-aware adversary and the oblivious adversaries. The data-aware adversary can increase the error linearly with $\epsilon$ using this strategy, while the oblivious adversaries fail to have any consistent impact on the resulting classifier's error with their strategies.

%% file: related_work_full.tex
\section{More Details on Related Work} \label{full_related_work}
As opposed to the data poisoning setting, the question of adversary's (adaptive) knowledge was indeed previously studied in the line of work on adversarial examples \citep{Adversarial:Old,Adversarial:models,Szegedy:intriguing}. 
In a test time evasion attack the adversary's goal is to find an adversarial example, the adversary knows the input $x$ \emph{entirely} before trying to find a close input $x'$   that is misclassified. So, this adaptivity aspect already differentiates adversarial examples from random noise. Moreover, the question of whether adversary knows the $\theta$ completely or it only has a black-box access to it \citep{papernot2017practical} also adds another dimension of adaptivity to the story. 

Some previous work have studied poisoning attacks in the setting of federated/distributed learning~\citep{bhagoji2018analyzing,mahloujifar97universal}. Their attacks, however, either (implicitly) assume a full  information attacker, or aim to increase the population risk (as opposed to injecting features in a feature selection task). Thus, our work is novel in both formally studying the differences between data-aware vs. data-oblivious attacks, and \emph{provably} separating the power of these two attack models in the contexts of feature selection.  Xiao et al.~\cite{xiao2015feature} also empirically examine the robustness of feature selection in the context of poisoning attacks, but their measure of stability is across sets of features. We are distinct in  that our paper studies the effect of data-oblivious attacks on \textit{individual} features and with provable guarantees.

Our work's motivation for data secrecy  might seem similar to other works that leverage privacy-preserving learning (and in particular differential privacy \citep{dinur2003revealing,TCC:DMNS06,dwork2008differential}) to limit the power of poisoning attacks by making the learning process less sensitive to poison data \citep{ma2019data}. However, despite seeming similarity, what we pursue here is fundamentally different. In this work, we try to understand the effect of keeping the data secret from adversaries. Whereas the robustness guarantees that come from differential privacy has nothing to do with secrecy and hold even if the adversary gets to see the full training set (or even select the whole training set in an adversarial way.).


We also distinguish our work with another line of work that studies the computational complexity of the attacker \citep{mahloujifar2018can,garg2019adversarially}. Here, we study the ``information complexity'' of the attack; namely, what information the attacker needs to succeed in a poisoning attack, while those works study the \emph{computational resources} that a poisoning attacker needs to successfully degrade the quality of the learned model. Another recent exciting line of work that studies the computational aspect of robust learning in poisoning contexts, focuses on the computational complexity of the \emph{learning} process itself \citep{diakonikolas2016robust,lai2016agnostic,charikar2017learning,diakonikolas2017statistical,diakonikolas2018list,diakonikolas2018sever,prasad2018robust,diakonikolas2018efficient}, and other works have studied the same question about the complexity of the learning process for evasion attacks \citep{bubeck2018adversarial1,bubeck2018adversarial2,degwekar2019computational}. Furthermore,   our work deals with information complexity and  is distinct from  works that study the impact of the training set (e.g., using clean labels) on the success of poisoning \cite{shafahi2018poison,zhu2019transferable,suciu2018does,turner2018clean}.

  Finally, we try to categorize the existing poisoning attacks in literature into data-oblivious and data-aware categories. The recent survey of \citep{goldblum2020data} and classifies existing poisoning attacks based on their techniques and goals. We use the same classes to categorize the attacks.
  
  \begin{itemize}
  \item{Feature Collision Attacks: [data-oblivious]} Feature Collision is a technique used in targeted poisoning attacks where the adversary tries to inject poison points around a target point $x$ so that the classification of $x$ is different than the correct label \citep{aghakhani2020bullseye,guopractical,shafahi2018poison,zhu2019transferable}. There is usually a ``clean label'' constraint for targeted attacks that prevents the adversary from using the same point as the target point. These attacks will be mostly categorized as data-oblivious as the attacker does not usually need to see the training set.
  \item{Bi-level Optimization Attacks: [data-aware]} Bi-level optimization is generic technique used for optimizing the poisoning points to achieve attacker's objective \citep{biggio2012poisoning,burkard2017analysis,geiping2020witches,hu2019targeted}. This optimization heavily relies on knowledge of training set. 
  \item{Label-Flipping Attacks: [both]} The idea of label-flipping is very simple yet effective. The random label-flipping attacks are data-oblivious as the only thing that the adversary does is to sample data from (conditional) distribution and flip the label. However, some variants of label-flipping \cite{zhao2017efficient,fung2018mitigating,biggio2011support} are relying on the training set to optimize the examples which makes them data-aware. 
  \item{Influence Function Attacks: [data-aware]} Attacks based on influence function look at the effect of training examples on the final loss of the model\cite{koh2017influencefunction,fang2020influence}. This technique require the knowledge of the training set.
  \item{Online Learning Attacks: [data-aware]}  \emph{Online} poisoning adversaries studied in~\cite{pTampTCC17,wang2018data,mahloujifar2019can}is a form of attack that lies somewhere between data-oblivious and data-aware attacks. In their model, an online adversary needs to choose its decision about the $i\th$ example (i.e., to tamper or not tamper it) based only on the history of the first $i-1$ examples, and without the knowledge of the future examples. So, their knowledge about the training data is limited, in a partial way. Since we separate the power of data-aware vs. data-oblivious attacks, a corollary of our results is that at least one of these models is different from the online variant for recovering sparse linear regression. In other words, we are in one of the following worlds: (i)  online adversaries are provably stronger than data-oblivious adversaries or (ii) data-aware  adversaries are provably stronger than online adversaries.
  \item{Federated Learning Attacks: [both]} The attack against federated learning \citep{bhagoji2018analyzing,tolpegin2020data,sun2020natural,cao2019understanding}, use a range of ideas that covers all the previous techniques and hence have both data-aware and data-oblivious variants. In general, since in federated learning the adversary sees the model updates at each round, they are more aware of the randomness of training process compared to typical poisoning attacks hence they can be more effective. 
  \end{itemize}
  

%% file: FullDefs.tex
\section{Further Details on Defining Oblivious Attacks} \label{sec:defDetails}
In this section, we discuss other definitional aspects of oblivious and full-information poisoning attacks.
\subsection{Oblivious Variants of (Data-aware) Data Poisoning Attacks} \label{sec:FeatureDetails}
In this section, we explain how to formalize oblivious poisoning attackers in general, and in the next subsection we will describe how to instantiate this general approach for the case of feature selection. 

A poisoning adversary of ``budget'' $k$, can tamper with a training sequence $\cS=\set{e_1,\dots,e_n}$, by ``modifying'' $\cS$ by at most $k$ changes. Such changes can be in three forms
\begin{itemize}
    \item {\bf Injection.} Adversary can inject $k$ new examples $e'_1,\dots,e'_k$ to $\cS$. This is without loss of generality when the learner is symmetric and is not sensitive to  the order in the training examples. More generally, when the training set is treated like a sequence $\cS = (e_1,\dots,e_n)$, the adversary can even choose the \emph{location} of these planted examples $e'_1,\dots,e'_k$. More formally, the adversary picks $k$ numbers $1\leq i_1 < \dots < i_k \leq n+k$, and  constructs the new data sequence $\cS'=(e''_1,\dots,e''_{n+k})$ by letting $e''_j=e'_{i_j}$ and letting $\cS$ fill the remaining coordinates of $\cS'$ in their original order from $\cS$.
    
    {\bf Oblivious injection.} In the full-information setting, the adversary can choose the poison examples and their locations based on  $\cS$. In the oblivious variant, the adversary chooses the poison examples $e'_1,\dots,e'_k$ and their locations $1\leq i_1 < \dots < i_k \leq n+k$ without knowing the original set $\cS$.
    
    \item {\bf Elimination.} Adversary can eliminate $k$ of the examples in $\cS$. When $\cS$ is a sequence, the adversary only needs to state the indexes $1 \leq i_1 < \dots , i_k \leq n$ of the removed examples.
    
    {\bf Oblivious elimination.}   In the full-information setting, the adversary can choose the locations of the deleted examples  based on  $\cS$. In the oblivious variant, the adversary chooses the locations  without knowing the original set $\cS$.
    
    \item {\bf Substitution and it oblivious variant.}
    These two settings are  similar to data elimination, with the difference that the adversary, in addition to the sequence of locations, chooses $k$ poison examples $e'_1,\dots,e'_k$ to substitute $e_{i_j}$ by $e'_j$ for all $j \in [k]$.
    
\end{itemize}

\paragraph{More general attack strategies.} One can think of more fine-grained variants of the substitution attacks above by having different "budgets" for injection and elimination processes (and even allowing different locations for eliminations and injections), but we keep the setting simple by default.

\subsection{Taxonomy for Attacks on Feature Selection}

Sometimes the goal of a learning process is to recover a model $\hat{\theta}$, perhaps from noisy data, that has the same set of features $\Supp(\hat{\theta})$ as the true model $\theta$. For example, those features could be the relevant factors determining a decease. Such process is called feature selection (or model recovery). A poisoning attacker attacking a feature selection task would directly try to counter this goal. Now, regardless of \emph{how} an attacker is transforming a data set $\cS$ into $\cS'$, let $\hat{\theta}'$ be the model that is learned from $\cS'$. Below we give a taxonomy of various attack scenarios.

\begin{compactitem}
    \item {\bf Feature adding.} In this case, the adversary's goal is to achieve $\Supp(\hat{\theta}') \not \se \Supp(\theta)$. Namely, adding a feature that is not present in the true model $\theta$.
    \item {\bf Feature removal.} In this case, the adversary's goal is to achieve $\Supp(\theta) \not \se \Supp(\hat{\theta}') $. Namely, removing a feature that is  present in the true model $\theta$.
    \item {\bf Feature flipping.} In this case, the adversary's goal is to do either of the above. Namely, $\Supp(\theta) \neq \Supp(\hat{\theta}') $, which means that at least one of the features' existence is flipped.
\end{compactitem}

\paragraph{Targeted variants of the attacks above.} For each of the three attack goals above (in the context of feature selection), one can envision a \emph{targeted} variant in which the adversary aims to add/remove or flip a specific feature $i \in [d]$ where $d$ is the data dimension.

%% file: general_lasso.tex
\section{Borrowed Results}\label{sec:gen_lasso}
In this section, we provide some preliminary results about the LASSO estimator. We first specify the sufficient conditions for a dataset that makes it a good dataset for robust recover using Lasso estimator. We borrow these specifications from the work of \cite{thakurta2013differentially}. We use these results in proving Theorem \ref{thm:main}. 

\begin{definition}[Typical systems] Suppose $\theta^*\in [0,1]^d$ be a model such that $|\Supp(\theta^*)|=s$. Let $X\in \R^{n\times d}$ and $Y\in \R^{n\times 1}$ and $W=Y-X\times \theta^*$. Also let $X_I\in \R^{n\times s}$ be a matrix formed by columns of $X$ whose indices are in $\Supp(\theta^*)$ and $X_O \in \R^{n\times (d-s)}$ be a matrix formed by columns of $X$ whose indices are not in $\Supp(\theta^*)$. The pair $(\theta^*, \concat{X}{Y})$ is called an \emph{$(n, d, s, \psi, \sigma)$-typical} system, if the following  hold:
\begin{itemize}
    \item{\textbf{Column normalization:}} Each column of $X$ has $\ell_2$ norm bounded by $\sqrt{n}$.
    \item{\textbf{Incoherence:}} $\norm{((X_O^TX_I)(X_I^T X_I)^{-1}\mathsf{sign}(\theta^*))}_\infty \leq 1/4.$
    \item {\textbf{Restricted strong Convexity:}} The minimum eigenvalue of $X_IX_I^T$ is at least $\psi$.
    \item {\textbf{Bounded noise}} $\norm{X_O^T(I_{n\times n}  - X_I(X_I^TX_I)^{-1}X_I^T)W}_\infty \leq 2\sigma\sqrt{n\log(d)}$.
\end{itemize}
\end{definition}\label{def:typical}
The following theorem is a modified version of result of \cite{wainwright2009sharp} borrowed from \cite{thakurta2013differentially}.
\begin{theorem}[Model recovery with Lasso \cite{wainwright2009sharp}]\label{thm:LassoRecoveryGen}
Let $(\theta^*,\concat{X}{Y})$ be a $(n,d,s,\sigma, \psi)$-typical system. Let $\alpha=\argmin_{i\in [d]} \max(\theta^*_i, 1 - \theta^*_i)$. If $n\geq 16\cdot\frac{\sigma}{\psi\cdot \alpha} \sqrt{s\cdot \log(d)}$ and then $\hat{\theta} = \Learn(\concat{X}{Y})$ would have the same support as $\theta^*$ when $\lambda = 4\sigma \sqrt{n\cdot \log(d)}$.
\end{theorem}

The following theorem is about robust model recovery with Lasso in \cite{thakurta2013differentially}.
\begin{theorem}[Robust model recovery with Lasso \cite{thakurta2013differentially}]\label{thm:LassoRobustRecoveryGen}
Let $(\theta^*,\concat{X}{Y})$ be a $(n,d,s,\sigma, \psi)$-typical system. Let $\alpha=\argmin_{i\in [d]} \max(\theta^*_i, 1 - \theta^*_i)$. If 
$$n\geq \max(\frac{16\sigma}{\psi\cdot \alpha} \sqrt{s\cdot \log(d)}, \frac{4s^4k^2(1/\psi +1)^2}{\log(d)\sigma^2})$$
then $\hat{\theta} = \Learn(\concat{X}{Y})$ would have the same support as $\theta^*$ when $\lambda = 4\sigma \sqrt{n\cdot \log(d)}$.

In addition, adding \emph{any} set of $k$ labeled vectors $\concat{X'}{Y'}$ with $\ell_\infty$ norm at most 1 to $\concat{X}{Y}$ would not change the support set of the model recovered by Lasso estimator. Namely,
\iffull
\begin{align*}
    \Supp\left(\Learn\left(\concatt{X}{Y}{X'}{Y'}\right)\right)&=\Supp(\Learn(\concat{X}{Y})
    =\Supp(\theta^*).
\end{align*}
\fi

Two theorems above are sufficient conditions for (robust) model recovery using lasso estimator. Bellow, we show two simple instantiating of the theorems on Normal distribution. Theorem bellow from the seminal work of Wainwright \cite{wainwright2009sharp} shows that the Lasso estimator with proper parameters provably finds the correct set of features, if the dataset and noise vectors are sampled from normal distributions. 
\end{theorem}
\bigskip
\begin{theorem}[\citep{wainwright2009sharp}]\label{thm:LassoRecovery}
Let $X$ be a dataset sampled from $\mathcal{N}(0,1/4)^{n\times d}$ 
and $W$ be a noise vector sampled from $\mathcal{N}(0,\sigma^2)^n$. For any $\theta^*\in(0,1)^d$ with at most $s$ number of non-zero coordinates, for $\lambda=4\sigma \sqrt{n\times \log(d)}$ and $n =\omega( s\cdot \log(d))$, with

probability at least $3/4$ 

over the choice of $X$ and $W$ (that determine $Y$ as well) 
we have $\Supp(\hat{\theta}) = \Supp(\theta^*)$  where $\hat{\theta} = \Learn(\concat{X}{Y}).$ Moreover, $\hat{\theta}$ is a unique minimizer for $\Loss(\cdot, \concat{X}{Y})$.
\end{theorem}

The above theorem requires the dataset to be sampled from a certain distribution and does not take into account the possibilities of outliers in the data. The robust version of this theorem, where part of the training data is chosen by an adversary, can be instantiated using Theorem \ref{thm:LassoRecoveryGen} as follows:

\begin{theorem} [\citep{thakurta2013differentially}]\label{thm:LassoRobustRecovery}
Let $X$ be a dataset sampled from $\mathcal{N}(0,1/4)^{n\times d}$ 
and $W$ be a noise vector sampled from $\mathcal{N}(0,\sigma^2)^n$. For any $\theta^*\in(0,1)^d$,
if $\lambda = 4\sigma\sqrt{n\times\log(d)}$ and $n= \omega(s \log(d) + s^4\cdot k^2)$, with probability at least $3/4$ 

over the choice of $X,W$ (determining $Y$), and $Y=X\times \theta^* + W$ it holds that, adding \emph{any} set of $k$ labeled vectors $\concat{X'}{Y'}$, such that rows of $X'$ has $\ell_\infty$ norm at most 1 and $Y$ has $\ell_\infty$ norm at most $s$, to $\concat{X}{Y}$ would not change the support set of the model recovered by Lasso estimator. Namely,
\iffull
\begin{align*}
    \Supp\left(\Learn\left(\concatt{X}{Y}{X'}{Y'}\right)\right)&=\Supp(\Learn(\concat{X}{Y})
    =\Supp(\theta^*).
\end{align*}
\fi

\end{theorem}

Note that Theorems \ref{thm:LassoRecovery} and \ref{thm:LassoRobustRecovery} are instantiation of Theorems \ref{thm:LassoRecoveryGen} and \ref{thm:LassoRobustRecoveryGen} for normal distribution and  are proved by showing that the sufficient conditions of those theorems will happen with high probability over the choice of dataset.

%% file: proof_main.tex
\section{Omitted Proofs}\label{proof:prop}
In this section, we prove Proposition \ref{prop:main} and Theorem \ref{thm:main}.

\subsection{Proof of Proposition \ref{prop:main}}
\begin{proof}
We first argue that winning the data-aware game of Definition \ref{security:Feature} is always possible. This is because, after getting the dataset $\concat{X}{Y}$ the adversary inspects the dataset to find out which coordinate is unstable and find a poisoning dataset that would add that unstable coordinate to the support set of the model.  

Now, we prove the other part of the proposition. That is, we show that no adversary can win the oblivious security game of Definition \ref{security:Feature} with probability more than $\epsilon$. 
The reason behind this claim is the $(k,\epsilon)$-resiliency of the dataset. For any fixed poisoning dataset $S'$ selected by adversary, the probability of $S'$ being successful in changing the support set is at most $\epsilon$. Therefore, the best strategy for an adversary that does not see the dataset is to pick the best possible poison dataset that maximizes the average success over all training data sampled from $D$, which we know is smaller than $\epsilon$ because of the resiliency. Note that, by averaging argument, randomness does not help the oblivious attack.

Therefore, the proof of Proposition \ref{prop:main} is complete.
\end{proof}

%% file: construction.tex
\subsection{Proof of Theorem \ref{thm:main}}\label{sec:optimal-attack}
Here, we outline the main lemmas that we need to prove Theorem \ref{thm:main}. We first some intermediate theorem and lemmas  that will be used to prove the main result. Then we prove these these intermidiate lemmas in the following subsection. 
 The following theorem shows an upper bound on the number of examples that a data-aware adversary need to add a  non-relevant feature to the support set of resulting model. 
 Before stating the Theorem, we define two useful notions.

\begin{definition}\label{def:alphabeta}
We define $$\alpha_i(\concat{X}{Y}) = X^T[i](Y- X\cdot \hat{\theta})$$
where $\hat{\theta} = \Lasso(\concat{X}{Y}).$
We also define $\beta_i$ similarly with difference that the minimization of Lasso is done in the subspace of vectors with the correct support. Namely,
$$\beta_i(\concatt{X}{Y}{X'}{Y'}) = X^T[i](Y- X\cdot \hat{\theta'})$$
where 
$\hat{\theta'} = \argmin _{\theta \in C} \frac{1}{n}\cdot\norm{\concatv{Y}{Y'}-\concatv{X}{X'}\times\theta}_2^2 + \frac{2\lambda}{n}\cdot \norm{\theta}_1.$ and $C$ is the subspace of models that their $i$th feature is 0 for all $i\not \in \Supp(\theta)$.
\end{definition}


\begin{theorem}[Unstability of Gaussian]\label{thm:attack}
Let $X \in \R^{n\times d}$ be an arbitrary matrix, $\theta^*\in [0,1]^d$ be an arbitrary vector,  $W$ be a noise vector sampled from $\mathcal{N}(0,\sigma^2)^{n\times 1}$, and let $Y=X\times \theta^* + W$. Also let $\lambda$ be the penalty parameter that is used for Lasso. Then for any $i$ there is a dataset $\concat{X'}{Y'}$ with at most $\lambda - |\alpha_i([\concat{X}{Y})|$ examples  
of $\ell_2$ norm at most $1$, such that 
$$i \in \Supp\left(\Learn\left(\concatt{X}{Y}{X'}{Y'}\right)\right).$$

\end{theorem}

Theorem above proves the existence of an attack that can add any feature to the training set. Below, we first provide the description of the attack. 

\paragraph{The Attack:} To attack a feature $i$ with $k$ examples, The attack first calculates $b=\Sign(\alpha_i(\concat{X}{Y}))$ use a dataset $\cS'=\concat{X'}{Y
 '}$ as follows:
\begin{align*}
    X' = \left[
    \begin{array}{cccc}
         0 & \dots & 1 & 0\\
         \vdots&\ddots &\vdots &\vdots\\
         0 & \dots & 1 & 0
    \end{array}
    \right] \in \mathbb{R}^{k \times d} 
    , 
    Y' = \left[
    \begin{array}{c}
          b\\
         \vdots\\
          b
    \end{array}
    \right] \in \mathbb{R}^{k \times 1}.
\end{align*}
The attack then adds $S'$ to the training set. Note that this attack is oblivious as it does not use the knowledge of the clean training set. This is the attack that we use in our experiments in Section \ref{sec:experiments}.

\begin{definition}[Re-sampling Operator] We define $R(X, I , \sigma)$ to be an operator that removes the $i$th column of $X$ and replace it with a fresh sample from $\Normal(0,\sigma^2)$ for all $i\in I$.
\end{definition}

\begin{theorem}[Resilience of Gaussian]\label{thm:failure_oblivious}
Let $[X',Y']$ be a dataset such that $|X'|_1 \leq k$ and let $S=\Supp(\Lasso(\concat{X}{Y})$ then we have
$$\Pr[ \Supp(\Lasso(\concatt{R(X,[d]\setminus S,\sigma)}{Y}{X'}{Y'})) \neq S ]\leq 2e^{-\frac{(\lambda-2k)^2}{2\sigma_2^2}}$$
where $\sigma^2_2 =\norm{(Y-\hat{\theta'}X)}^2_2 \cdot \sigma^2\leq (n+k)\sigma^2.$
\end{theorem}
Theorem above states that if we re-sample the $i$th coordinate of $X$, then the probability of $[X',Y']$ being successful in adding $i$th feature to support set is limited.

Lastly, we state a lemma that shows a lower bound on the error of the lasso estimator. This Lemma will be used in analyzing the power of data-aware adversary.

\begin{lemma}\label{lem:largeerror}
Let $\hat{\theta}=\Lasso(\concat{X}{Y})$ and $w=\norm{Y-X\hat{\theta}}_2$. Also assume for each column of $X$ we have $\norm{X^T[i]}_2\leq L$. then we have,
$$w\geq \frac{\lambda}{L}.$$
\end{lemma}

\paragraph{Putting things together} Now we put things together to complete the proof of Theorem \ref{thm:main}. For the oblivious adversary, by Theorem \ref{thm:failure_oblivious}, the probability of the oblivious attacker succeeding according to Theorem \ref{thm:attack_opt} is bounded by probability $ 2e^{-\frac{(\lambda-2k)^2}{2(n+k)\sigma^2}}$. This means, setting $\lambda=2k+\sigma\sqrt{2(n+k)\log(2/\epsilon_2)}$ will guarantee that the oblivious attacker will succeed with probability at most $\epsilon_2$.
For the data-aware adversary, consider the distribution $\R(X,\set{i},\sigma)[i](Y-X\hat{\theta})$. We know that this distribution is a Gaussian distribution with standard deviation $w\sigma$ for $w=\norm{Y-\hat{\theta}X}_2$. 
Therefore, by Theorem \ref{thm:attack}, and Gaussian tail bound, we know that with probability at least $p_1\geq 1- (1- 2e^{-2\frac{(\lambda-k)^2}{w\sigma^2}})^{d-s}$ over the choice of randomness on the $i$th column, the data-aware adversary will succeed by just doing succeed in adding a feature to the support set.  Also, using Lemma \ref{lem:largeerror}, we can show that this probability is larger than $1- (1- 2e^{-2\frac{(\lambda-k)^2L^2}{\lambda^2\sigma^2}})^{d-s}$.  Now, we can  set $ d=s+\frac{\log(1-\eps_1)}{\log(1- 2e^{-2\frac{L^2(\lambda-k)^2}{{\lambda}^2\sigma^2}})}$  so that the oblivious adversary succeeds with probability at least $\eps_1$.

 \subsection{Proofs of Theorems \ref{thm:attack},
 \ref{thm:attack_opt} and
 \ref{thm:failure_oblivious} and Lemmas  \ref{lemma:Gaus} and \ref{lem:largeerror}}
 \input{proofs_construction}

%% file: proofs_construction.tex
\iffull
We first state and prove the following useful lemma.
\fi
\begin{lemma}\label{lem:1}
Let $X\in \R^{n\times d}$ and $Y\in \R^n$.
Let $\hat{\theta}$ be a vector that minimizes  
$\Loss(\cdot, \concat{X}{Y})$.
Then, for all non-zero coordinates $j \in [d]$, where $\hat{\theta}_j \neq 0$ we have

$$\sum_{i=1}^n X_{(i,j)}\cdot(Y_i-\langle\hat{\theta}, X_i\rangle) = -\lambda\cdot \Sign(\hat{\theta}_j),$$
and for all $0$ coordinates $j\in [d]$, where $\theta_j=0$, we have
$$\left|\sum_{i=1}^n X_{(i,j)}\cdot(Y_i-\langle\hat{\theta}, X_i\rangle)\right| < \lambda.$$
\end{lemma}
\iffull
\begin{proof}[Proof of Lemma~\ref{lem:1}]
Since $\hat{\theta}$ is a minimizer of $f(\cdot)$, the derivative of $f$ should be 0 or undefined on all coordinates at $\hat{\theta}$. 
Note that, for all non-zero coordinates $i$ the derivative of the second term $2\lambda\norm{\theta}_1$ is equal to $2\lambda\Sign(\theta_i)$. Therefore, for non-zero coordinates the derivative of the first term should be equal to $-2\lambda\cdot \Sign(\theta_i)$. That is,
$$2(X^T \times(Y-X\times \hat{\theta}))_i=2\lambda\cdot \Sign(\theta_i)$$
which proves the first part of the lemma. For the second part, note that the derivative of $f$ does not exist, but the left-hand and right-hand derivatives exist and $\hat{\theta}$ minimizes $f$. Therefore, the left-derivative should be negative and the right hand derivative should be positive. Thus, we have
$$2(X^T \times(Y-X\times \hat{\theta}))_i + 2\lambda >0,$$ 
and
$$2(X^T \times(Y-X\times \hat{\theta}))_i - 2\lambda <0,$$
which implies that
$$-\lambda <(X^T \times(Y-X\times \hat{\theta}))_i <\lambda,$$
finishing the proof of the lemma.
\end{proof}
\fi
Now we state an analytical lemma that helps us bound the effect of an oblivious adversary in increasing the $\ell_\infty$ norm of a Gaussian distribution by adding a predetermined vector to it.
\begin{lemma}\label{lem:erf}
Define $f_{L,\sigma}(x)=\frac{erf(\frac{L+x}{\sigma})+erf(\frac{L-x}{\sigma})}{2erf(\frac{L}{\sigma})}.$ For any $a \in R$ and $b \in R$ we have $f(a)f(b) > f(|a|+|b|).$
\end{lemma}
\begin{proof}
Define $g(x)=\log(f_{L,\sigma}(x)).$ It is easy to check that $g$ is a concave function with the property that $|x|g'(|x|)\leq g(x).$ Assume $|b|<|a|$,  we have
$$g(|a|+|b|) \leq g(|a|) + |b|g'(|a|)\leq g(a) + |b|g'(|b|) \leq g(a)+g(b).$$
\end{proof}
\begin{corollary}\label{cor:gaussianstrategy}
Let $a=R^d$ be a vector such that $|a|_1=l$ and let $b\equiv \N(0,\sigma^2)^d$. We have, $\Pr[|b+a|_\infty > r] \leq 2e^{\frac{-(r-l)^2}{2\sigma^2}}.$ 
\end{corollary}
\begin{proof}
This follows from Lemma \ref{lem:erf} by writing the exact probability using the CDF of Gaussian and then applying a Gaussian tail bound.
\end{proof}
Now we state another theorem that shows a lower bound on the number of poisoning points required to add a specific feature.

\begin{theorem}\label{thm:attack_opt}
Let $\concat{X'}{Y'}$ be such that $$i \in \Supp\left(\Learn\left(\concatt{X}{Y}{X'}{Y'}\right)\right)$$
and 
$$
i \not \in \Supp\left(\Learn\left(\concat{X}{Y}\right)\right)$$
then for some $j\not \in  \Supp\left(\Learn\left(\concat{X}{Y}\right)\right)$ we have 
$$2\norm{X'^T[j]}_1 \geq \lambda - \beta_j(\concatt{X}{Y}{X'}{Y'}).$$
\end{theorem}
\begin{proof}
Consider $\hat{\theta'}$ to be the optimal model on the subspace defined by the support of $\hat{\theta}$. If $\concat{X'}{Y'}$ adds feature $i$ to the support set, then by uniqueness, $\hat{\theta'}$ cannot be a solution. This means that the sub-gradients of $\hat{\theta'}$ should not satisfy the properties of Lemma \ref{lem:1}.  The only thing the adversary can do is to violate the condition on of the coordinates that are not in support. In particular, for some $j$, the $j$th coordinate must have
$$\left|\sum_{i=1}^{n+k} \concatv{X}{X'}_{(i,j)}\cdot(\concatv{Y}{Y'}_i-\langle\hat{\theta'}, \concatv{X}{X'}_i\rangle)\right| \geq\lambda.$$
Therefore, by the norm constraint of the last $k$ columns we have
$$\left|\sum_{z=1}^{n} X_{(z,j)}\cdot(Y_z-\langle\hat{\theta'}, X_z\rangle)\right| \geq\lambda-2\norm{X'^T[j]}_1.$$
\end{proof}

Now we state a Lemma that shows how $\beta_i$ is distributed, when re-sampling the $i\th$ column of the matrix.

\begin{lemma}\label{lemma:Gaus}
Consider $\concatt{X}{Y}{X'}{Y'}$, for any $i\in[d]$ and set $I$ such that $i\in I$,  we have 
$$\beta_i(\concatt{R(X,I,\sigma)}{Y}{X'}{Y'})\equiv \Normal(0, \sigma_2^2)$$

where 
$\sigma^2_2 =\norm{(Y-\hat{\theta'}X)}^2_2 \cdot \sigma^2\leq (n+k)\sigma^2$ for $\hat{\theta'}$ of Definition \ref{def:alphabeta}.
\end{lemma}
\begin{proof}
We have
$$\beta_i(\concatt{R(X,i,\sigma)}{Y}{X'}{Y'})\equiv \sum_{i=1}^{n} (Y-\hat{\theta'}X)[i]\cdot \Normal(0,\sigma^2)\equiv \Normal(0,\norm{(Y-\hat{\theta'}X)}^2_2\sigma^2) .$$
We know that $$\norm{(Y-\hat{\theta'}X)}^2_2\leq (n+k)s^2$$ because $\theta'$ minimizes the criterion and should lead to a smaller loss than a model with $0$ everywhere. 
\end{proof}
We are now ready to Prove our Theorems \ref{thm:attack} and $\ref{thm:failure_oblivious}$.
\begin{proof}[Proof of Theorem~\ref{thm:attack}] Let $k\geq \lambda - |\alpha_i(\concat{X}{Y})|$ and
 consider $X'$ which is a $k\times d$ matrix that is $0$ everywhere except on the $i\th$ column that is 1 and $Y'$ is a $k\times1$ vector that is equal to $b=\Sign(\alpha_i(\concat{X}{Y})$ everywhere. We show that by adding this matrix the adversary is able to add $i\th$ coordinate to the support set of the $\hat{\theta'} = \Learn\left(\concatt{X}{Y}{X'}{Y'}\right)$. To prove this, suppose the $i\th$ coordinate of $\hat{\theta}'$ is $0$. Thus, we have
\iffull
\begin{align*}
    &\left(\concatv{X}{X'}^T\times\left(\concatv{Y}{Y'}-\concatv{X}{X'}\times\hat{\theta}'\right)\right)_i 
    =  kb + \left(X^T\times(Y-X\times\hat{\theta}')\right)_i. \numberthis \label{eqn1}
\end{align*}
\fi

Now we prove that $\hat{\theta}'$ also minimizes the Lasso loss over $\concat{X}{Y}$. This is because for any vector $\theta$ with $i\th$ coordinate $0$, we have
$$\Loss\left(\theta, \concatt{X}{Y}{X'}{Y'}\right) = kb + \Loss(\theta, \concat{X}{Y}).$$ 
Now, let $\hat{\theta}$ be the minimizer of $\Loss(\cdot, \concat{X}{Y})$. We know that $\hat{\theta}$ is $0$ on the $i\th$ coordinate. Therefore we have,
\iffull
\begin{align}
    \Loss\left(\hat{\theta}, \concatt{X}{Y}{X'}{Y'}\right) &= kb + \Loss\left(\hat{\theta}, \concat{X}{Y}\right)\nonumber\\
    &\geq \Loss\left(\hat{\theta}', \concatt{X}{Y}{X'}{Y'}\right)
    = kb + \Loss(\hat{\theta}', \concat{X}{Y}).\label{ineq:p01}
\end{align}
\fi
where the last inequality comes from the fact that $\hat{\theta'}$ minimizes the loss over $\concatt{X}{Y}{X'}{Y'}$. On the other hand, we know that 
\begin{align}\label{ineq:p02}
    \Loss(\hat{\theta}', \concat{X}{Y})\geq \Loss(\hat{\theta}, \concat{X}{Y})
\end{align}
because $\hat{\theta}$ minimizes $\Loss(\cdot,\concat{X}{Y})$. Inequalities \ref{ineq:p01} and \ref{ineq:p02} imply that 
$$\Loss(\hat{\theta}, \concat{X}{Y}) = \Loss(\hat{\theta}', \concat{X}{Y})$$
and that $\hat{\theta}$ minimizes $\Loss(\cdot, \concatt{X}{Y}{X'}{Y'})$. Therefore, based on Lemma~\ref{lem:1}, since the $i\th$ coordinate of $\hat{\theta}$ is zero we have
\begin{align*}
    \left|(\concatv{X}{X'}^T \times(\concatv{Y}{Y'}-\concatv{X}{X'}\times \hat{\theta}))_i \right| < \lambda\numberthis \label{eqn2}.
\end{align*}
However,  by definition of $\alpha$ we have
$$\left|\concatv{X}{X'}^T\left(\concatv{Y}{Y'}-\concatv{X}{X'}\times \hat{\theta}\right)_i\right| = |\alpha_i(\concat{X}{Y}) + \Sign(\alpha_i(\concat{X}{Y}))\cdot k| \geq \lambda.$$

This is a contradiction. Hence, the $i\th$ coordinate could not be $0$ and the proof is complete.
\end{proof}

Now we prove Theorem \ref{thm:failure_oblivious}.
\begin{proof}[Proof of Theorem \ref{thm:failure_oblivious}]
Let $r_j=|X'[j]|$ and vector $r=(2r_1,\dots,2r_d)$. also define vector $\beta=(\beta_1,\dots,\beta_d)$. According to Theorem \ref{thm:attack_opt}, we know that $|(r+\beta)|_\infty\geq \lambda$ must hold. On the other hand, by Lemma \ref{lemma:Gaus} we know that $\beta$ is distributed according to a Gaussian distribution with standard deviation $\sigma_2$. Therefore, by Corollary $\ref{cor:gaussianstrategy}$ we can bound the probability of success of the adversary by
$ 2e^{-\frac{(\lambda-2k)^2}{2\sigma_2^2}}.$
\end{proof}
We now finish this section by proving Lemma \ref{lem:largeerror}.
\begin{proof}[Proof of Lemma \ref{lem:largeerror}]
Consider an index $j\in \Supp(\hat{\theta})$. By Cauchy-Schwarz inequality we have
$$(\sum_{i=1}^n(Y_i-\langle\hat{\theta}, X_i\rangle)^2)(\sum_{i=1}^n X_{(i,j)}^2)\geq (\sum_{i=1}^n X_{(i,j)}\cdot(Y_i-\langle\hat{\theta}, X_i\rangle))^2.$$
By Lemma \ref{lem:1} we have 
$$(\sum_{i=1}^n X_{(i,j)}\cdot(Y_i-\langle\hat{\theta}, X_i\rangle)^2 = \lambda^2$$
Therefore,
$$w^2 L^2 \geq \lambda^2.$$
\end{proof}